\documentclass{article}

\usepackage{microtype}
\usepackage{graphicx}
\usepackage{subfigure}
\usepackage{booktabs} 
\usepackage{multirow}
\usepackage{tikz}
\usepackage{threeparttable}
\usepackage{tikz-cd,mathtools}
\usepackage{mathtools}
\usetikzlibrary{arrows, matrix} 

\usepackage{hyperref}

\usepackage[accepted]{icml2024}

\usepackage{amsmath}
\usepackage{amssymb}
\usepackage{mathtools}
\usepackage{amsthm}

\usepackage[capitalize,noabbrev]{cleveref}

\theoremstyle{plain}
\newtheorem{theorem}{Theorem}[section]

\newtheorem{lemma}[theorem]{Lemma}

\theoremstyle{definition}

\theoremstyle{remark}
\newtheorem{remark}[theorem]{Remark}

\usepackage[textsize=tiny]{todonotes}

\icmltitlerunning{Transolver: A Fast Transformer Solver for PDEs on General Geometries}

\begin{document}

\twocolumn[
\icmltitle{Transolver: A Fast Transformer Solver for PDEs on General Geometries}




\begin{icmlauthorlist}
\icmlauthor{Haixu Wu}{software}
\icmlauthor{Huakun Luo}{software}
\icmlauthor{Haowen Wang}{software}
\icmlauthor{Jianmin Wang}{software}
\icmlauthor{Mingsheng Long}{software}
\end{icmlauthorlist}

\icmlaffiliation{software}{School of Software, BNRist, Tsinghua University. \\Haixu Wu $<$wuhx23@mails.tsinghua.edu.cn$>$}

\icmlcorrespondingauthor{Mingsheng Long}{mingsheng@tsinghua.edu.cn}

\icmlkeywords{Scientific machine learning, PDE solver}

\vskip 0.3in
]



\printAffiliationsAndNotice{}  

\begin{abstract}
Transformers have {empowered many milestones} across various fields and have recently been applied to solve partial differential equations (PDEs). However, since PDEs are typically discretized into large-scale meshes with complex geometries, {it is challenging for Transformers to capture intricate physical correlations directly from massive individual points}.
Going beyond superficial and unwieldy meshes, we present Transolver based on a more foundational idea, which is learning intrinsic physical states {hidden} behind discretized geometries. {Specifically}, we propose {a new} Physics-Attention to adaptively split the discretized domain into a series of learnable slices {of} flexible shapes, where mesh points under similar physical states will be ascribed to the same slice. By calculating attention to physics-aware tokens encoded from slices, Transovler can effectively capture intricate physical correlations under complex geometrics, which also empowers {the solver} with endogenetic geometry-general modeling capacity and can be efficiently computed in linear complexity. Transolver achieves consistent state-of-the-art with 22\% relative gain across six standard benchmarks and also excels in large-scale industrial simulations, including car and airfoil designs. {Code is available at \href{https://github.com/thuml/Transolver}{https://github.com/thuml/Transolver}}.
\end{abstract}

\section{Introduction}
Solving partial differential equations (PDEs) is of immense importance in extensive real-world applications, such as weather forecasting, industrial design, and material analysis \citep{roubivcek2013nonlinear}. As a basic scientific problem, it is usually hard to obtain analytic solutions for PDEs. Thus, {PDEs} are typically discretized into meshes and then solved by numerical methods in practice, which usually takes a few hours or even days for complex structures \cite{umetani2018learning}. Recently, deep models have emerged as promising tools for solving PDEs \citep{lu2021learning,li2021fourier}. Benefiting from their impressive non-linear modeling capacity, they can learn to approximate the input and output mappings of PDE-governed tasks from data during training and then infer the solution significantly faster than numerical methods at the inference phase \citep{goswami2022deep,wu2023LSM}.

As {the major backbone} of foundation models, Transformers \citep{NIPS2017_3f5ee243} have achieved remarkable processes in extensive areas \citep{Devlin2019BERTPO,NEURIPS2020_1457c0d6,dosovitskiy2021an,liu2021Swin}, which have also been introduced in PDE solving \cite{li2023transformer}. However, as PDEs are usually discretized into large-scale meshes with complex geometrics for precise simulation, directly applying Transformers to massive mesh points faces difficulties in both computational efficiency and relation learning \cite{liu2021Swin,katharopoulos2020transformers}, impeding them from being ideal PDE solvers. For instance, to calculate the drag force of a driving car (Figure \ref{fig:intro}), the model needs to approximate the solution of Navier-Stokes equations, including estimating the pressure for surface meshes and velocity for surrounding volumes, which poses the following two challenges. First, this problem involves collaborative modeling of tens of thousands of irregularly placed mesh points, which is computationally prohibited for canonical attention due to its quadratic complexity. Second, PDEs involve extremely complex spatiotemporal interactions among multiple physics quantities. It is hard to capture these high-order and intricate correlations directly from massive individual points. Thus, \emph{how to efficiently capture physical correlations {underlying} the discretized domain} is the key to {``transform''} Transformers into practical PDE solvers.

\begin{figure*}[h]
\begin{center}
\centerline{\includegraphics[width=\textwidth]{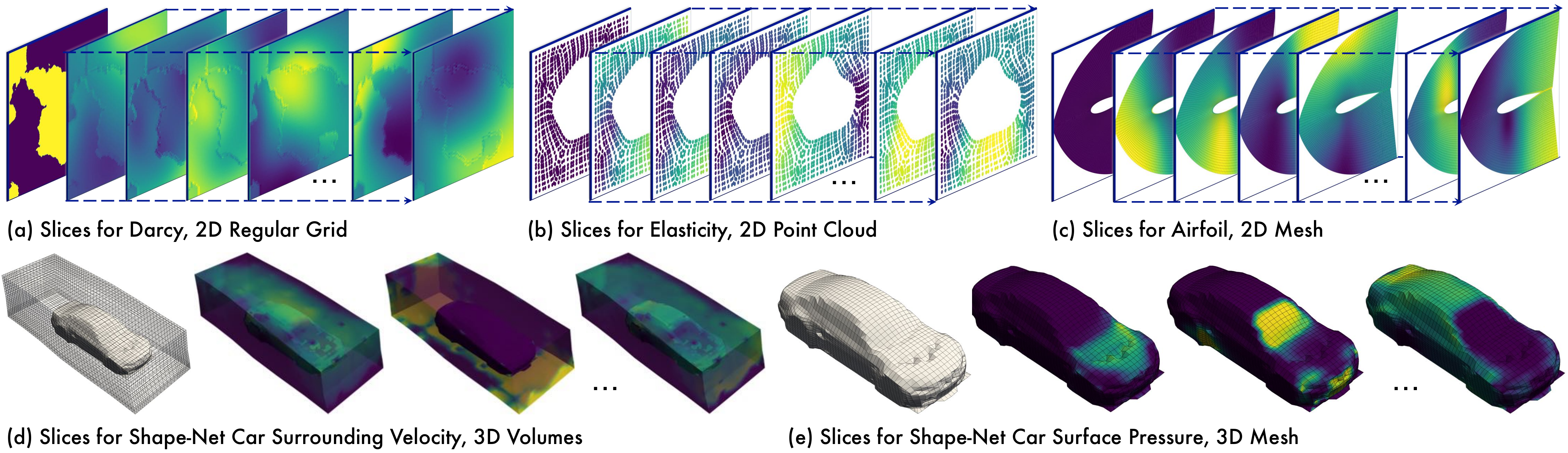}}
    \vspace{-5pt}
	\caption{Visualization of learned slices in Transolver. For each case, the leftmost subfigure is model input and the right shows learned slices. A brighter color indicates the mesh point is more ascribed to the corresponding slice. See Appendix \ref{appdix:full_vis} for more visualizations.}
	\label{fig:intro}
\end{center}
\vspace{-20pt}
\end{figure*}

Previous methods attempt to tackle the complexity problem by introducing linear attention \citep{hao2023gnot,anonymous2023factorized}, but directly applying the attention to massive mesh points may overwhelm the model from learning informative relations \cite{wu2022flowformer}. In addition, solely relying on features of individual points is also insufficient in capturing intricate physical correlations of PDEs \cite{trockman2022patches}, especially for industrial design, which usually involves extremely complex multiphysics interactions. Besides, although the patchify operation is widely adopted to augment the feature of a single pixel with local information in Vision Transformers \cite{dosovitskiy2021an,liu2021Swin}, the regular shape of patches is not applicable {to} unstructured geometries, let alone captures complicated physical states hidden under various discretized meshes.

Seeing through superficial and unwieldy meshes, this paper presents Transolver based on a more foundational idea, which is learning intrinsic physical states under complex geometrics. We present the Physics-Attention to decompose the discretized domain into a series of learnable slices, where mesh points under similar physical states will be ascribed to the same slice and then encoded into a physics-aware token. By applying attention to these learned physics-aware tokens, Physics-Attention can effectively capture complex underlying interactions behind the discretized domain. As shown in Figure~\ref{fig:intro}, the learned slices clearly reflect miscellaneous \emph{physical states} of PDEs, such as various fluid-structure interactions in a Darcy flow, different extrusion regions of elastic materials, shock wave and wake flow around the airfoil, front-back surfaces and up-bottom spaces of driving cars. This design can also natively adapt to intricate geometries and be effectively computed in linear time. We conduct extensive experiments on six well-established benchmarks with various geometries and large-scale industrial simulations, where Transolver achieves consistent state-of-the-art with impressive relative gain. Overall, our contributions are summarized as follows:
\vspace{-5pt}
\begin{itemize}
    \item Beyond prior methods, we propose to solve PDEs by learning intrinsic physical states behind the discretized domain, which frees our model from complex meshes and allows it to focus more on physical interactions.
    \item We present Transolver with Physics-Attention to decompose discretized domain into a series of learnable slices and apply attention to encoded physics-aware tokens, which can be computed in linear complexity.
    \item Transolver achieves consistent state-of-the-art with 22\% relative gain across six standard benchmarks and {excels in large-scale industrial simulations (e.g.~car and airfoil designs), presenting favorable efficiency, scalability and out-of-distribution generalizability.}
\end{itemize}


\section{Related Work}

\subsection{Neural PDE Solvers} As a long-standing {foundational} problem {in science and engineering}, solving PDEs has gained significant attention. In the past centuries, various classical numerical methods, such as finite element method and spectral methods, have been proposed and widely used in practical applications \cite{Wazwaz2002PartialDE,solin2005partial}. Recently, in view of the remarkable non-linear modeling capacity, deep models have also been introduced for solving PDEs as a fast surrogate \cite{karniadakis2021physics,wang2023scientific}, which can be roughly categorized into the following two paradigms. 

\vspace{-5pt}
\paragraph{Physics-informed neural networks} This paradigm formalizes PDE constraints, including equations, initial and boundary conditions as objective functions of deep models \cite{Weinan2017TheDR,Raissi2019PhysicsinformedNN,Wang2020UnderstandingAM,Wang2020WhenAW}. During training, the output of deep models will gradually satisfy PDE constraints, which can successfully approximate the PDE solution. However, this paradigm requires the exact formalization of PDEs, thus usually hard to apply to partially observed real-world applications.

\vspace{-5pt}
\paragraph{Neural operators} Another paradigm is to learn neural operators to approximate the input-output mappings in PDE-governed tasks, such as predicting the future fluid based on past observations or estimating the inner stress of solid materials \cite{lu2021learning,jmlr_operator}. The most well-established models are FNO \cite{li2021fourier} and its variants. \citet{li2021fourier} proposed Fourier neural operators by approximating integration with linear projection in the Fourier domain. Afterward, U-NO \cite{rahman2022u} and U-FNO \cite{Wen2021UFNOA} are presented by plugging the FNO with U-Net \cite{ronneberger2015u} for multiscale modeling. Besides, WMT \cite{Gupta2021MultiwaveletbasedOL} introduces multiscale wavelet bases to capture the complex correlations {at} various scales. F-FNO \cite{anonymous2023factorized} enhances model efficiency by employing factorization on the Fourier domain. LSM \cite{wu2023LSM} is recently proposed to tackle the high-dimension complexity of PDEs by applying spectral methods \cite{gottlieb1977numerical} in learned latent space.

To tackle irregular meshes, GNO \cite{Li2020NeuralOG} employs graph neural operators, and geo-FNO \cite{li2023geometry} utilizes the geometric Fourier transform to project the irregular input domain into uniform latent mesh. Recently, GINO \cite{li2023geometryinformed} combines GNO and geo-FNO for local and global simultaneous modeling. 3D-GeoCA \cite{anonymous2023geometryguided} enhances GNO by incorporating pre-trained 3D vision backbones \cite{xue2023ulip} {as better model initializations}. However, due to the periodic boundary assumption of Fourier bases \cite{gottlieb1977numerical}, geo-FNO will degenerate seriously in complex meshes, e.g.~a car shape. Graph kernels also fall short in learning global information.

Especially, Transformers \cite{NIPS2017_3f5ee243}, as a {vital cornerstone} of deep learning, have also been applied to solve PDEs. HT-Net \cite{anonymous2023htnet} integrates Swin Transformer \cite{liu2021Swin} and multigrid method \cite{wesseling1995introduction} to capture the multiscale spatial correlations. FactFormer \cite{li2023scalable} utilizes the low-rank structure to boost the model efficiency with multidimensional factorized attention. However, these methods assume that PDEs are discretized into a uniform grid, limiting their applications in unstructured meshes. Besides, to address the quadratic complexity of attention, OFormer \cite{li2023transformer}, GNOT \cite{hao2023gnot} and ONO \cite{anonymous2023improved} utilize the well-established linear Transformers, such as Reformer \cite{kitaev2020reformer}, Performer \cite{choromanski2021rethinking} and Galerkin Transformer \cite{Cao2021ChooseAT}. Still, all of these methods directly apply attention to massive mesh points. In contrast, Transolver applies attention to intrinsic physical states captured by learnable slices, thereby {better adept at} modeling intricate physical correlations.

\vspace{-4pt}
\subsection{Geometric Deep Learning}
A series of techniques have been developed to handle irregular geometrics, named geometric deep learning \cite{bronstein2017geometric}. Graph neural networks are representative ones, which employ kernels on connected graphs for representation learning \cite{hamilton2017inductive,gao2019graph,pfaff2021learning}. Besides, PointNet \cite{qi2017pointnet} and Point Transformer \cite{zhao2021point} are also presented for scatter point clouds. However, most of these methods are proposed for computer vision or graphics, which are different from the PDE-solving task in this paper. Besides, all of these methods are well-designed for complex geometrics, while Transolver, {benefiting} from learning {physical-sensitive} slices, is {apt at capturing physics information underlying} unwieldy meshes.

\section{Method}
To tackle difficulties in efficiency and correlation modeling, we present Transolver with Physics-Attention to learn {high-level} correlations among intrinsic physical states under discretized meshes in PDE-governed tasks. Different from learning {low-level} relations over mesh points, focusing on physical states will free our model from complex geometrics, benefiting physics {solving} and computation efficiency. 

\vspace{-5pt}
\paragraph{Problem setup} Consider PDEs defined on {input} domain $\Omega\subset\mathbb{R}^{C_\mathbf{g}}$, where ${C_\mathbf{g}}$ denotes the dimension of input space. For numerical calculation, $\Omega$ is firstly discretized into a finite set of $N$ mesh points $\mathbf{g}\in\mathbb{R}^{N\times {C_\mathbf{g}}}$. The task is to estimate target physical quantities based on input geometrics $\mathbf{g}$ and quantities $\mathbf{u}\in\mathbb{R}^{N\times {C_\mathbf{u}}}$ observed on $\mathbf{g}$.
Here $\mathbf{u}$ is optional in some PDE-governed tasks. For instance, for fluid prediction, the input includes both the observation grid and observed past fluid velocity, where the target is the future velocity on each grid point. As for car or airfoil designs, the input only contains the discretized mesh structure and the model needs to estimate the surface and surrounding physics quantities.

\subsection{Learning Physics-Aware Tokens}

As we discussed before, the key to solving PDEs is to capture intricate physical correlations. However, the numerous discretized mesh points may overwhelm the attention mechanism from learning reliable correlations. Seeing through superficial meshes, we find that these mesh points are a finite {discrete} sampling of the underlying {continuous} physics space, which inspires us to learn {the} intrinsic physical states. As shown in Figure~\ref{fig:slice}, which is to estimate the surface pressure of a driving car, we notice that the surface mesh set can be ascribed to several physically internal-consistent subsets, such as front, bevel and back areas. This {discovery} provides a more foundational view for solving PDE-governed tasks.

\begin{figure}[h]
\begin{center}
\centerline{\includegraphics[width=\columnwidth]{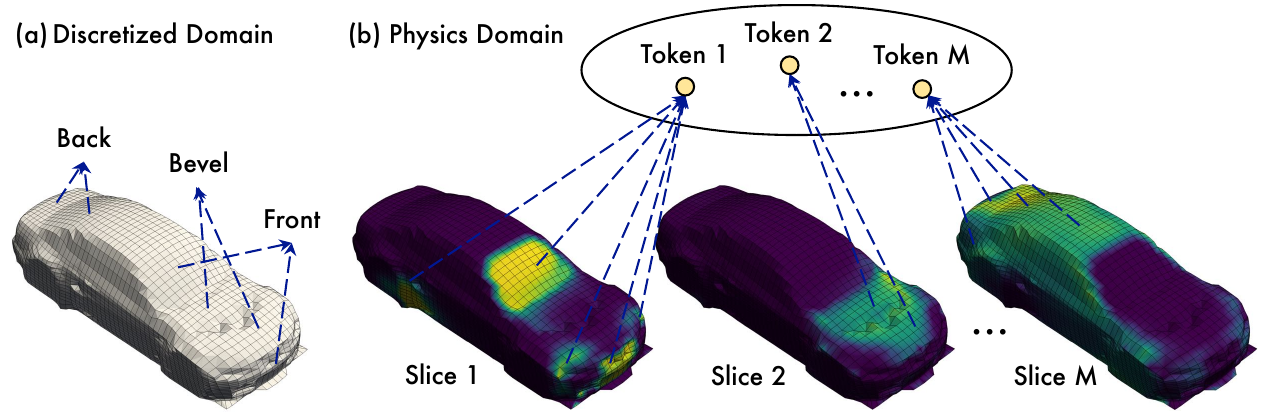}}
    \vspace{-5pt}
	\caption{Learning physics-aware tokens from Transolver slices.}
	\label{fig:slice}
\end{center}
\vspace{-20pt}
\end{figure}

\begin{figure*}[t]
\begin{center}
\centerline{\includegraphics[width=\textwidth]{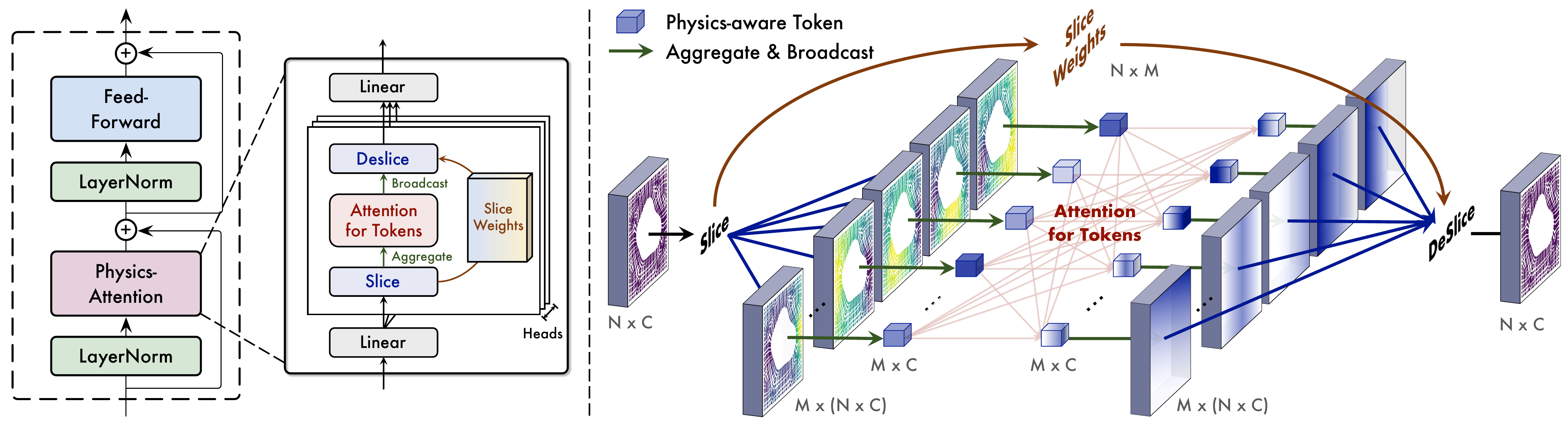}}
    \vspace{-5pt}
	\caption{Overall design of Transolver layer, which replaces the standard attention with Physics-Attention. Each head encodes the input domain into a series of physics-aware tokens and then captures physical correlations under intricate geometrics by attention among tokens.}
	\label{fig:slice_attn}
\end{center}
\vspace{-20pt}
\end{figure*}

Technically, given a mesh set $\mathbf{g}=\{\mathbf{g}_{i}\}_{i=1}^{N}$ with the coordinate information of $N$ mesh points and observed quantities $\mathbf{u}$, we firstly embed them into deep features $\mathbf{x}=\{\mathbf{x}_{i}\}_{i=1}^{N}$ by a linear layer, where each mesh point feature contains $C$ channels, i.e.~$\mathbf{x}_{i}\in\mathbb{R}^{1\times C}$, and involves both geometry and physics information. To capture physical states under the whole input domain, we propose a bottom-up paradigm, that ascribes each mesh point $\mathbf{g}_{i}$ to $M$ potential slices based on its learned feature $\mathbf{x}_{i}$, which is formalized as follows:
\begin{equation}
	\begin{split}\label{equ:learn_slice}
\{\mathbf{w}_{i}\}_{i=1}^{N} &= \left\{\operatorname{Softmax}\big(\operatorname{Project}\left(\mathbf{x}_{i}\right)\big)\right\}_{i=1}^{N} \\
\mathbf{s}_{j} &= \left\{\mathbf{w}_{i,j}\mathbf{x}_{i}\right\}_{i=1}^{N},
	\end{split}
\end{equation}
where $\operatorname{Project}()$ projects $C$ channels into $M$ weights and yields slice weights $\mathbf{w}_{i}\in\mathbb{R}^{1\times M}$ after $\operatorname{Softmax}()$. Specifically, $\mathbf{w}_{i,j}$ represents the degree that the ${i}$-th mesh point belongs to the $j$-th slices with $\sum_{j=1}^{M} \mathbf{w}_{i,j}=1$. $\mathbf{s}_{j}\in\mathbb{R}^{N\times C}$ represents the $j$-th slice feature, which is a {weighted combination of $N$ mesh point features $\mathbf{x}$}. Note that mesh points with close features will derive similar slice weights, which means they are more likely to be assigned to the same slice. To avoid a uniform assignment of each mesh point, we adopt $\operatorname{Softmax}()$ along the slice dimension {(i.e.~newly projected $M$ dimension)} to make learned slice weights {low-entropy} and ensure informative physical states. In practice, $\operatorname{Project}()$ is configured as a point-wise linear layer, which can naturally adapt to general geometries. As for structured meshes or uniform grid, it can also be instantiated as a local convolution, mesh-free layer for better representations.

Afterward, since each slice contains mesh points with similar geometry and physics features, we further encode them into physical-aware tokens by spatially weighted aggregation, which can be written as follows:
\begin{equation}
	\begin{split}\label{equ:token_encoding}
\mathbf{z}_{j} = \frac{\sum_{i=1}^{N} \mathbf{s}_{j,i}}{\sum_{i=1}^{N}\mathbf{w}_{i,j}} = \frac{\sum_{i=1}^{N} \mathbf{w}_{i,j}\mathbf{x}_{i}}{\sum_{i=1}^{N} \mathbf{w}_{i,j}} ,
	\end{split}
\end{equation}
where $\mathbf{z}_{j}\in\mathbb{R}^{1\times C}$. We normalize {each} token {feature $\mathbf{z}_j$} by dividing the sum of slice weights.
After encoding from physically internal-consistent slices by spatial aggregation, each token contains information {of} a specific physical state.

\vspace{2pt}
\begin{remark}[\textbf{Why slices can learn physically internal-consistent information}] Firstly, as we aforementioned, slice weights are projected from mesh features. Thus, mesh points with similar features will be more likely to be assigned to the same slice. Secondly, since we will apply attention to the tokens encoded from slices, to decrease the final loss, the slice weights will be further optimized to assign mesh points under similar physical states to the same slice during training. Otherwise, the attention among tokens could be confused by the less distinguishable and {state-hybrid} token features, resulting in a less satisfying performance.
\end{remark}
\vspace{2pt}
\begin{remark}[\textbf{Learning slice is different from splitting computation area}] Classical numerical methods, such as finite element method, usually split the whole mesh into several computation areas for better simulation. This process requires huge specialized knowledge and manual effort \cite{solin2005partial} and can only cover spatially local areas. It is insufficient to capture points under similar physical states but spatially distant, e.g.~windshield and license plate of driving cars. In this paper, we take benefits from deep features and learn physical states in a bottom-up paradigm. The learned slices are beyond local areas. As shown in Figure~\ref{fig:intro}(e), the model learns to ascribe the windshield, license plate and headlight of the car into the same slice because they are all in the front area during driving, which is highly related to the drag force, verifying the effectiveness of learning slices.
\end{remark}





\subsection{Transolver}
Based on the idea of learning physics-aware tokens, we propose the Transolver by renovating Transformer with Physics-Attention to capture intricate physical correlations of PDEs.

\vspace{-5pt}
\paragraph{Physics-Attention} As described in the last section, for a deep feature $\mathbf{x}\in\mathbb{R}^{N\times C}$ embedded from input, we firstly decompose it into $M$ physically internal-consistent slices $\mathbf{s}=\{\mathbf{s}_{j}\}_{j=1}^{M}\in\mathbb{R}^{M\times (N\times C)}$ based on learned slice weights $\mathbf{w}\in\mathbb{R}^{N\times M}$. Then, to obtain the specific physics information contained in each slice, we aggregate $M$ slices to $M$ physics-aware tokens $\mathbf{z}=\{\mathbf{z}_{j}\}_{j=1}^{M}\in\mathbb{R}^{M\times C}$ by Eq.~\eqref{equ:token_encoding}.

Next, as shown in Figure \ref{fig:slice_attn}, we employ the attention mechanism among encoded tokens to capture intricate correlations among different physical states, that is
\begin{equation}
\begin{split}\label{equ:attn}
\mathbf{q}, \mathbf{k}, \mathbf{v} = \operatorname{Linear}(\mathbf{z}), \ \ \mathbf{z}^\prime = \operatorname{Softmax}\left(\frac{\mathbf{q}\mathbf{k}^{\sf T}}{\sqrt{C}}\right)\mathbf{v},
	\end{split}
\end{equation}
where $\mathbf{q}, \mathbf{k}, \mathbf{v}, \mathbf{z}^\prime\in\mathbb{R}^{M\times C}$.
Afterward, transited physical tokens $\mathbf{z}^\prime=\{\mathbf{z}_{j}^\prime\}_{j=1}^{M}$ are transformed back to mesh points by deslicing, which recomposes tokens with slice weights:
\begin{equation}
	\begin{split}\label{equ:deslice}
\mathbf{x}_{i}^\prime & = \sum_{j=1}^{M} \mathbf{w}_{i,j}\mathbf{z}_{j}^\prime,
	\end{split}
\end{equation}
where $1\leq i\leq N$ and each token $\mathbf{z}_{j}^\prime$ is broadcasted to all mesh points during above calculation.
For clarity, we summarize the above process as $\mathbf{x}^\prime = \operatorname{Physics-Attn}(\mathbf{x})$, whose overall complexity is $\mathcal{O}(NMC + M^2 C)$. Since we set $M$ as a constant and $M\ll N$, the computation complexity is linear w.r.t.~the number of mesh points. Following the convention of attention mechanism \cite{NIPS2017_3f5ee243}, we adopt the multi-head version for Physics-Attention to augment the model capacity, which splits the input feature into several subspaces along the channel dimension.

\vspace{2pt}
\begin{remark}[\textbf{Attention as learnable integral operator}] Prior methods define the PDE-solving task as an iterative updated process \cite{Li2020NeuralOG} and prove that canonical attention is a Monte-Carlo approximation of the integral operator on the input domain $\Omega$ \cite{Cao2021ChooseAT,jmlr_operator}, which can be used to approximate the solving process of each iteration step. However, in our work, the attention is applied to tokens encoded from slices. Toward a better theoretical understanding of Physics-Attention, we will prove that our design is also equivalent to learnable integral on $\Omega$.
\end{remark}

\begin{theorem}[\textbf{Physics-Attention is equivalent to learnable integral on $\Omega$}]
\label{theorem:understanding}
Given input function $\boldsymbol{u}:\Omega\to\mathbb{R}^{C}$ and a mesh point $\mathbf{g}^\ast\in\Omega$, Physics-Attention is to approximate the integral operator $\mathcal{G}$, which is defined as:
\begin{equation}
	\begin{split}\label{equ:integral}
\mathcal{G}(\boldsymbol{u})(\mathbf{g}^\ast) & = \int_{\Omega}\kappa(\mathbf{g}^\ast,\boldsymbol{\xi})\boldsymbol{u}(\boldsymbol{\xi})\mathrm{d}\boldsymbol{\xi},
	\end{split}
\end{equation}
where $\kappa(\cdot,\cdot)$ denotes the kernel function defined on $\Omega\times \Omega$.
\end{theorem}
\begin{proof} By constructing a diffeomorphism projection between mesh domain $\Omega$ and slice domain $\Omega_{\mathrm{s}}$, and substituting integration variable from $\boldsymbol{\xi}\in\Omega$ to $\boldsymbol{\xi}_{\mathrm{s}}\in\Omega_{\mathrm{s}}$, we can rewrite Eq.~\eqref{equ:integral} as an integral on $\Omega_{\mathrm{s}}$, which is approximated by attention in Eq.~\eqref{equ:attn}. 
See Appendix \ref{appendix:proof} for complete proof.
\end{proof}

\paragraph{Overall design} Following the architecture of canonical Transformer \cite{NIPS2017_3f5ee243}, we propose the Transolver by replacing the attention mechanism with Physics-Attention. Suppose there are $L$ layers, as shown in Figure \ref{fig:slice_attn}, the $l$-th layer of Transolver can be formalized as follows:
\begin{equation}
	\begin{split}\label{equ:overall}
\hat{\mathbf{x}}^{l} &= \operatorname{Physics-Attn}\left(\operatorname{LayerNorm}\left({\mathbf{x}}^{l-1}\right)\right) + {\mathbf{x}}^{l-1}\\
{\mathbf{x}}^{l} &= \operatorname{FeedForward}\left(\operatorname{LayerNorm}\left(\hat{\mathbf{x}}^{l}\right)\right) + \hat{\mathbf{x}}^{l},
	\end{split}
\end{equation}
where $l\in\{1,\cdots, L\}$. ${\mathbf{x}}^{l}\in\mathbb{R}^{N\times C}$ is the output of the $l$-th layer. $\mathbf{x}^{0}\in\mathbb{R}^{N\times C}$ represents the input deep feature, which is embedded from input geometries $\mathbf{g}\in\mathbb{R}^{N\times C_{\mathbf{g}}}$ and initial observation $\mathbf{u}\in\mathbb{R}^{N\times C_{\mathbf{u}}}$ by a linear embedding layer, i.e.~$\mathbf{x}^{0}=\operatorname{Linear}\left(\operatorname{Concat}(\mathbf{g}, \mathbf{u})\right)$. Here $C_{\mathbf{g}}$ is the dimension of geometry space and $C_{\mathbf{u}}$ is the number of observed physical quantities. At last, we adopt a linear projection {upon} $\mathbf{x}^{L}$ and obtain the final output {as predictions of $\mathbf{u}$}.


\section{Experiments}
We conduct extensive experiments to evaluate Transolver, including six well-established benchmarks and two industrial-level design tasks, covering various geometrics.

\vspace{-5pt}
\paragraph{Benchmarks} As presented in Table \ref{tab:dataset}, our experiments span point cloud, structured mesh, regular grid and unstructured mesh in both 2D and 3D space. Elasticity, Plasticity, Airfoil, Pipe, Navier-Stokes and Darcy {were} proposed by FNO \cite{li2021fourier} and geo-FNO \cite{Li2022FourierNO}, which have been widely followed. Besides, we also experiment with car and airfoil design tasks. Shape-Net Car \cite{umetani2018learning} is to estimate the surface pressure and surrounding air velocity given vehicle shapes. AirfRANS \cite{bonnet2022airfrans} contains high-fidelity simulation data for Reynolds-Averaged Navier–Stokes equations on airfoils from the National Advisory Committee for Aeronautics.

\begin{table}[h]
\vspace{-10pt}
	\caption{Summary of experiment benchmarks, which includes various geometrics. \#Mesh records the size of discretized meshes.}
	\label{tab:dataset}
	\vskip 0.1in
	\centering
	\begin{small}
		\begin{sc}
			\renewcommand{\multirowsetup}{\centering}
			\setlength{\tabcolsep}{4.7pt}
			\scalebox{1}{
			\begin{tabular}{l|c|c|c}
				\toprule
			    Geometry & Benchmarks & \#Dim & \#Mesh \\
			    \midrule
                 Point Cloud & Elasticity & 2D & 972 \\
                 \midrule
			     Structured  & Plasticity & 2D+Time & 3,131 \\
        	Mesh & Airfoil & 2D & 11,271 \\
                  & Pipe & 2D & 16,641 \\
                  \midrule
                  \multirow{2}{*}{Regular Grid} 
		          & Navier–Stokes & 2D+Time & 4,096 \\	     
                    & Darcy & 2D & 7,225 \\
                \midrule
                Unstructured & Shape-Net Car & 3D & 32,186 \\
                Mesh & AirfRANS & 2D & 32,000 \\
				\bottomrule
			\end{tabular}}
		\end{sc}
	\end{small}
	\vspace{-5pt}
\end{table}

\vspace{-5pt}
\paragraph{Baselines} We comprehensively compare Transolver with more than 20 baselines, including typical neural operators: FNO \citeyearpar{li2021fourier}, U-NO \citeyearpar{rahman2022u}, LSM \citeyearpar{wu2023LSM}, etc, Transformer PDE solvers: GNOT \citeyearpar{hao2023gnot}, FactFormer \citeyearpar{li2023scalable}, etc, and classical geometric deep models: PointNet \citeyearpar{qi2017pointnet}, GraphSAGE \citeyearpar{hamilton2017inductive}, MeshGraphNet \citeyearpar{pfaff2021learning}, etc. LSM \cite{wu2023LSM} and GNOT \cite{hao2023gnot} are previous state-of-the-art {on} standard benchmarks. GINO \cite{li2023geometryinformed} and 3D-GeoCA \cite{anonymous2023geometryguided} are advanced deep models {for} large-scale industrial-level simulation benchmarks.

\begin{table*}[t]
	\caption{Performance comparison on standard benchmarks. Relative L2 is recorded. A smaller value indicates better performance. For clarity, the best result is in bold and the second best is underlined. Promotion refers to the relative error reduction w.r.t.~the second best model ($1-\frac{\text{Our error}}{\text{The second best error}}$) on each benchmark. ``/'' means that the baseline cannot apply to this benchmark.}
	\label{tab:mainres_standard}
	\vspace{-5pt}
	\vskip 0.15in
	\centering
	\begin{small}
		\begin{sc}
			\renewcommand{\multirowsetup}{\centering}
			\setlength{\tabcolsep}{7.2pt}
			\begin{tabular}{l|cccccc}
				\toprule
                    \multirow{3}{*}{Model} & \multicolumn{1}{c}{Point Cloud} & \multicolumn{3}{c}{Structured Mesh} & \multicolumn{2}{c}{Regular Grid} \\
                    \cmidrule(lr){2-2}\cmidrule(lr){3-5}\cmidrule(lr){6-7}
				& Elasticity & Plasticity & Airfoil & Pipe & Navier–Stokes & Darcy \\
				\midrule
                    FNO \citep{li2021fourier} & / & / & / & / & 0.1556 & 0.0108 \\
                    WMT \citep{Gupta2021MultiwaveletbasedOL} & 0.0359 & 0.0076 & 0.0075 & 0.0077 & {0.1541} & 0.0082 \\
                    U-FNO \citep{Wen2021UFNOA} &0.0239 & 0.0039 & 0.0269 & {0.0056} & 0.2231 & 0.0183 \\
                    geo-FNO \citep{Li2022FourierNO} & {0.0229} & 0.0074 & 0.0138 & 0.0067 & 0.1556 & 0.0108 \\
                    U-NO \citep{rahman2022u} & 0.0258 & {0.0034} & 0.0078 & 0.0100 & 0.1713 & 0.0113 \\
                    F-FNO \citep{anonymous2023factorized} &0.0263 & 0.0047 & 0.0078 & 0.0070 & 0.2322 & {0.0077} \\
                    LSM \citep{wu2023LSM} & 0.0218 & 0.0025 & \underline{0.0059} & 0.0050 & 0.1535 & \underline{0.0065} \\
                    \midrule
                    Galerkin \citep{Cao2021ChooseAT} & 0.0240 & 0.0120 & 0.0118 & 0.0098 & 0.1401 & 0.0084 \\
                    HT-Net \citep{anonymous2023htnet} & / & 0.0333 & 0.0065 & 0.0059 & 0.1847 & 0.0079 \\
                    OFormer \citep{li2023transformer} & 0.0183 & \underline{0.0017} & 0.0183 & 0.0168 & 0.1705 & 0.0124 \\
                    GNOT \citep{hao2023gnot} & \underline{0.0086} & 0.0336 & 0.0076 & \underline{0.0047} & 0.1380 & 0.0105 \\
                    FactFormer \citep{li2023scalable} & / & 0.0312 & 0.0071 & 0.0060 & 0.1214 & 0.0109 \\
                    ONO \citep{anonymous2023improved} & 0.0118 & 0.0048 & 0.0061 & 0.0052 & \underline{0.1195} & 0.0076 \\
                    \midrule
                    \textbf{Transolver (Ours)} & \textbf{0.0064} & \textbf{0.0012} & \textbf{0.0053} & \textbf{0.0033} & \textbf{0.0900} & \textbf{0.0057} \\
                    Relative Promotion & 25.6\% & 29.4\% & 10.2\% & 29.7\% & 24.7\% & 12.3\% \\
				\bottomrule
			\end{tabular}
		\end{sc}
	\end{small}
    \vspace{-5pt}
\end{table*}

\vspace{-5pt}
\paragraph{Implementations} For fairness, we set the number of layers $L$ as 8 and the channel of hidden features $C$ as 128 or 256 according to the number of observed quantities of input data, which ensures that our model parameter is comparable to other Transformer-based models, such as GNOT \citeyearpar{hao2023gnot} or ONO \citeyearpar{anonymous2023improved}. The number of slices $M$ is chosen from $\{32,64\}$ according to the hidden dimension to balance model efficiency. All the experiments are conducted on one NVIDIA A100 GPU and repeated three times. In addition to the relative L2 of estimated physics fields, we also calculate the error and Spearman’s rank correlation of drag and lift coefficients for practical design tasks. See Appendix \ref{appendix:detail} for comprehensive descriptions of implementations.

\vspace{3pt}
\subsection{Main Results}
\vspace{2pt}
\paragraph{Standard Benchmarks} To clearly benchmark our model among multifarious neural operators, we first experiment on six well-established datasets, which can conveniently build a complete leaderboard from previous papers \cite{Li2022FourierNO,wu2023LSM,hao2023gnot}.

As presented in Table \ref{tab:mainres_standard}, Transolver achieves consistent state-of-the-art across six widely-used benchmarks, covering solid and fluid physics in various geometrics. Notably, Transolver gains significant promotion in tasks on point cloud and structured mesh (25.6\% in Elasticity, 29.4\% in Plasticity, 29.7\% in Pipe), demonstrating the effectiveness of our design in handling complex geometrics. Also, some advanced Transformer-based models, such as OFormer and GNOT, directly apply linear attention to mesh points, where we can find that they fall short in handling the Darcy benchmark. This is because Darcy requires the model to simulate the fluid pressure through the porous medium, which involves complex \emph{fluid-structure interaction} \cite{bungartz2006fluid} along the twisty medium boundary, while directly applying attention to massive mesh points will degenerate in correlation modeling \cite{wu2022flowformer}. {Benefiting from learning physical states, Transolver can obtain more informative geometry tokens than single mesh points and significantly reduce the token number, thereby capturing complex fluid-structure interactions better.}

\vspace{-5pt}
\paragraph{Practical Design} Wind tunnel testing is one of the most essential steps in industrial design. As described in Figure~\ref{fig:pratical_design}, to examine the model effectiveness in complex real-world applications, we also experiment with the simulated wind tunnel scenario. Concretely, we simulate a driving car and record the surface pressure and surrounding air velocity, which can be used to calculate drag force. The task is to estimate the surface and surrounding physics fields based on the car's surface geometry. As for the AirfRANS, in addition to different airfoil shapes, our dataset also includes different angles of attack with different Reynolds numbers, which can better cover real flying cases. Note that these two tasks are quite challenging since they require the model to handle multiphysics simulation for {hybrid} geometrics.

\begin{figure}[t]
\begin{center}
\centerline{\includegraphics[width=0.85\columnwidth]{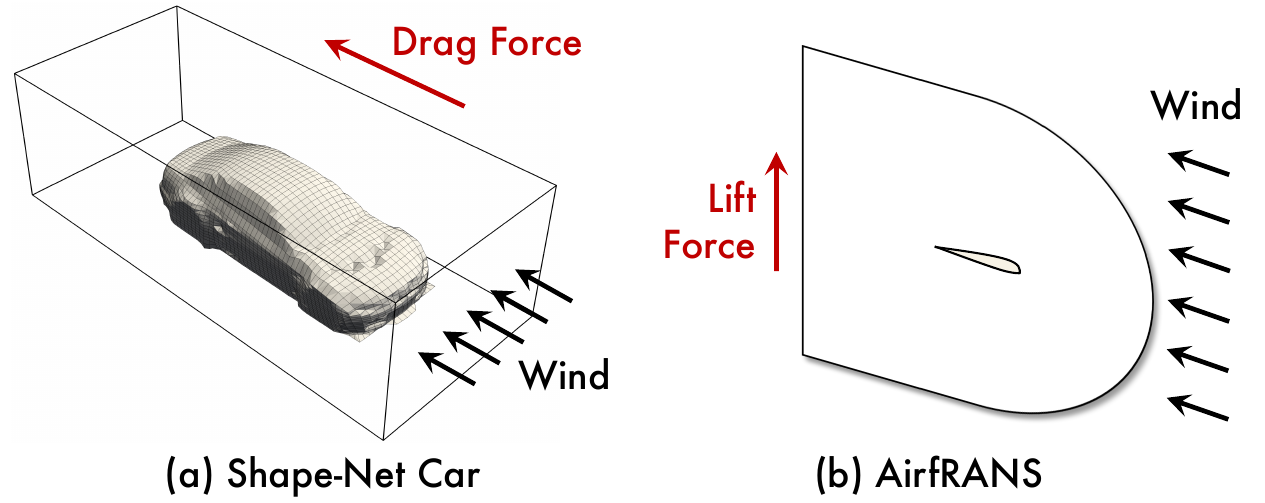}}
    \vspace{-5pt}
	\caption{Car and airfoil design tasks. The key problem is to estimate the drag and lift force of a driving car or a flying airplane.}
	\label{fig:pratical_design}
\end{center}
\vspace{-30pt}
\end{figure}

\begin{table*}[t]
	\caption{Performance comparison on practical design tasks. Both benchmarks are {with} unstructured mesh. In addition to the relative L2 of the surrounding (Volume) and surface (Surf) physics fields, the relative L2 of drag coefficient ($C_{D}$) and lift coefficient ($C_{L}$) is also recorded, along with their Spearman’s rank correlations $\rho_{D}$ and $\rho_{L}$. A Spearman's correlation close to 1 indicates better performance.}
	\label{tab:mainres_design}
	\vspace{-5pt}
	\vskip 0.15in
	\centering
	\begin{small}
		\begin{sc}
			\renewcommand{\multirowsetup}{\centering}
			\setlength{\tabcolsep}{4.8pt}
			\begin{tabular}{l|cccccccccc}
				\toprule
                    \multirow{3}{*}{Model$^\ast$} & \multicolumn{4}{c}{Shape-Net Car} & \multicolumn{4}{c}{AirfRANS} \\
                    \cmidrule(lr){2-5}\cmidrule(lr){6-9}
				& Volume $\downarrow$  & Surf $\downarrow$ & $C_{D}$ $\downarrow$ & $\rho_{D}$ $\uparrow$ & Volume $\downarrow$ & Surf $\downarrow$ & $C_{L}$ $\downarrow$ & $\rho_{L}$ $\uparrow$  \\
                    \midrule
                    Simple MLP & 0.0512 & 0.1304 & 0.0307 & 0.9496 & 0.0081 & 0.0200 & 0.2108 & 0.9932  \\
                    GraphSAGE \citep{hamilton2017inductive} & 0.0461 & 0.1050 & 0.0270 & 0.9695 & 0.0087 & 0.0184 & \underline{0.1476} & \underline{0.9964} \\
                    PointNet \citep{qi2017pointnet} & 0.0494 & 0.1104 & 0.0298 & 0.9583 & 0.0253 & 0.0996 & 0.1973 & 0.9919 \\
                    Graph U-Net \citep{gao2019graph} & 0.0471 & 0.1102 & 0.0226 & 0.9725 & 0.0076 & 0.0144 & 0.1677 & 0.9949 \\
                    MeshGraphNet \citep{pfaff2021learning} & 0.0354 & 0.0781 & 0.0168 & 0.9840 & 0.0214 & 0.0387 & 0.2252 & 0.9945 \\
                    \midrule
                    GNO \citep{li2020neural} & 0.0383 & 0.0815 & 0.0172 & 0.9834 & 0.0269 & 0.0405 & 0.2016 & 0.9938 \\
                    Galerkin \citep{Cao2021ChooseAT} & 0.0339 & 0.0878 & 0.0179 & 0.9764 & 0.0074 & 0.0159 & 0.2336 & 0.9951 \\
                    geo-FNO \citep{Li2022FourierNO} & \textcolor{gray}{0.1670} & \textcolor{gray}{0.2378} & \textcolor{gray}{0.0664} & \textcolor{gray}{0.8280} & \textcolor{gray}{0.0361} & \textcolor{gray}{0.0301} & \textcolor{gray}{0.6161} & \textcolor{gray}{0.9257} \\
                    GNOT \citep{hao2023gnot} & 0.0329 & 0.0798 & 0.0178 & 0.9833 & \underline{0.0049} & \underline{0.0152} & 0.1992 & 0.9942 \\
                    GINO \citep{li2023geometryinformed} & 0.0386 & 0.0810 & 0.0184 & 0.9826 & 0.0297 & 0.0482 & 0.1821 & 0.9958 \\
                    3D-GeoCA \citep{anonymous2023geometryguided} & \underline{0.0319} & \underline{0.0779} & \underline{0.0159} & \underline{0.9842} & / & / & / & / \\
                    \midrule
                    \textbf{Transolver (Ours)} & \textbf{0.0207} & \textbf{0.0745} & \textbf{0.0103} & \textbf{0.9935} & \textbf{0.0037} & \textbf{0.0142} & \textbf{0.1030} & \textbf{0.9978} \\
				\bottomrule
			\end{tabular}
		\end{sc}
        \begin{tablenotes}
      \footnotesize
      \item[] $\ast$ Since both datasets are {with} unstructured mesh, not all the baselines are applicable. Concretely, the capability to handle unstructured mesh of typical neural operators (e.g., U-NO, LSM, etc) is based on geo-FNO \citeyearpar{Li2022FourierNO}, which degenerates a lot in complex geometrics. Some Transformer-based models will also come across unstable training due to the massive mesh points, such as ONO and OFormer.
      \end{tablenotes}
	\end{small}
    \vspace{-5pt}
\end{table*}

As shown in Table \ref{tab:mainres_design}, we can find that Transolver also excels in these two complex tasks among various geometric deep models and neural operators. In addition to achieving more accurate physics field estimation in both volume and surface, Transolver also performs best in design-oriented metrics, including drag and lift coefficient, as well as Spearman's rank correlation. Note that this metric measures the correlation between the ranking distribution of real coefficients and model-estimated values on all test samples, which quantifies the model's ability to rank different car or airfoil designs, {therefore} especially essential to shape optimization.

It is also worth noticing that geo-FNO degenerates seriously in Shape-Net Car. This is because geo-FNO transforms the input geometry into a uniform latent grid based on Fourier bases, which is insufficient for surface-volume {hybrid} geometry. Besides, 3D-GeoCA enhances GNO by employing the advanced 3D geometric deep model Point-BERT \cite{yu2022point} as the feature encoder. However, even empowered with the advanced 3D geometric encoder, it still underperforms the Transolver. These results further demonstrate the advantages of our model in solving PDEs.

\begin{table}[t]
    \vspace{-5pt}
	\caption{Ablations on Physics-Attention. We experiment on two variants: changing the number of slices $M$ and replacing the learnable slices with fixed regular squares in $4\times 4$ size (\#Regular Squares). Efficiency is calculated on inputs with 1024 unstructured mesh points and batch size as 1. See Appendix \ref{appdix:ablations} for full ablations.}
	\label{tab:ablation}
	\vskip 0.1in
	\centering
	\begin{small}
		\begin{sc}
			\renewcommand{\multirowsetup}{\centering}
			\setlength{\tabcolsep}{2.3pt}
			\scalebox{1}{
			\begin{tabular}{l|c|cc|cc}
				\toprule
			\multicolumn{2}{c|}{\multirow{2}{*}{Ablations}} & \scalebox{0.95}{\#Memory} & \scalebox{0.95}{\#Time} & \multicolumn{2}{c}{\scalebox{0.95}{Relative L2}} \\
                \multicolumn{2}{c|}{} & \scalebox{0.8}{(GB)} & \scalebox{0.8}{(s/epoch)} & \scalebox{0.95}{Elasticity} & \scalebox{0.95}{Darcy} \\
			    \midrule
                   & {1} & 0.60 & 37.76 & 0.0148 & 0.0386 \\
                   & {8} & 0.60 & 37.82 & 0.0071 & 0.0096 \\
                   & {16} & 0.61  & 37.96 & 0.0067 & 0.0067 \\
                   & 32 & 0.62 & 38.00 & 0.0067 & 0.0063  \\
                  Number & 64 & 0.64 & 38.18 & 0.0064 & 0.0059 \\
                  of Slices & 96 & 0.68 & 38.31 & 0.0061 & 0.0055  \\
                  & 128 & 0.69 & 38.78 & 0.0058 & 0.0054 \\
                  & {256} & 0.81 & 39.13 & \textbf{0.0054} & \textbf{0.0050}  \\
                  & {512} & 1.01 & 39.75 & 0.0059 & 0.0056 \\
                  & {1024} & 1.53 & 40.49 & 0.0068 & 0.0055  \\
                  \midrule
                  \multicolumn{2}{l|}{{\scalebox{0.8}{Regular Squares}}} & / & / & / & 0.0088 \\
				\bottomrule
			\end{tabular}}
		\end{sc}
	\end{small}
	\vspace{-10pt}
\end{table}

\vspace{-5pt}
\paragraph{Ablations} In addition to the main results, we also include ablations of our design in Table \ref{tab:ablation}. In general, we can find that increasing the number of slices will benefit the final performance, which enables the model to capture more fine-grained physical states but also brings more computation costs. To balance efficiency and performance, we set $M$ as 64 for Elasticity and Darcy and choose it from $\{32,64\}$ for other benchmarks. {Note that when we set $M=1$, the Physics-Attention will degenerate to a global pooling operator, which omits all the physical correlations and will cause a serious performance drop. Also, we can further observe that an extremely large $M$ (e.g.~1024) will slightly decrease the final performance. This may be because a too-large $M$ will make the physics domain seriously fragmented, resulting in too many tokens for subsequent attention calculation. In principle, the best choice of slice number is up to the physical property of the target PDE. In our experiments, $M$ is easy-to-tune in the range from 32 to 256.}

Besides, replacing learnable slices with fixed regular squares will damage the model performance seriously, even for Darcy which is originally discretized into regular grids. This result further demonstrates the advantages of capturing physical states over solely computing in the discretized domain.

\begin{figure*}[t]
\begin{center}
\centerline{\includegraphics[width=\textwidth]{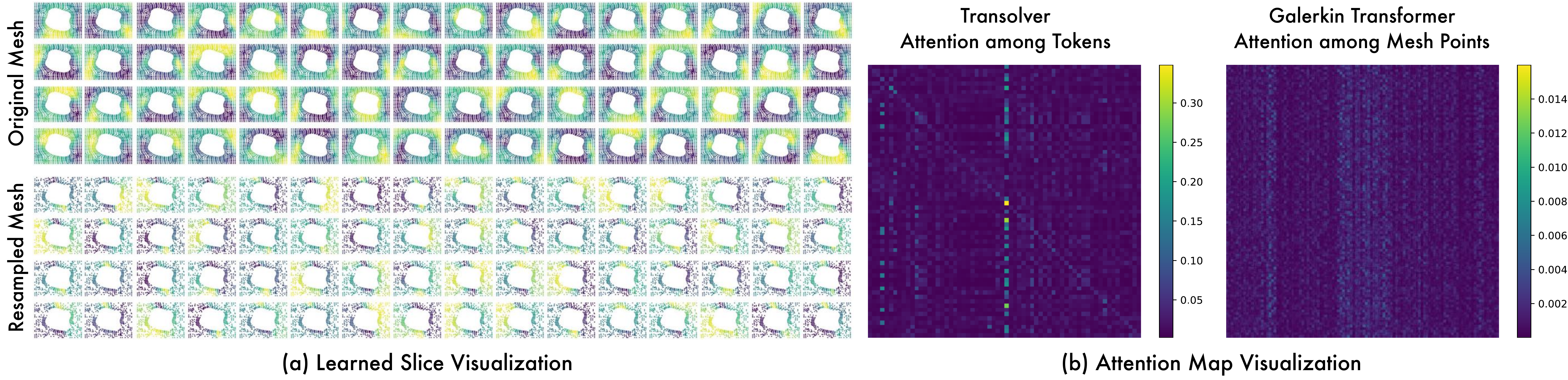}}
    \vspace{-5pt}
	\caption{Physics-Attention visualization on Elasticity: (a) slice weights in the last layer of Transolver for both original and resampled meshes, (b) attention maps of the last layer in Transolver and Galerkin Transformer \cite{Cao2021ChooseAT}. See Appendix \ref{appdix:full_vis_slice} for more visualizations.}
	\label{fig:visual}
\end{center}
\vspace{-20pt}
\end{figure*}

\begin{figure}[t]
\begin{center}
\centerline{\includegraphics[width=\columnwidth]{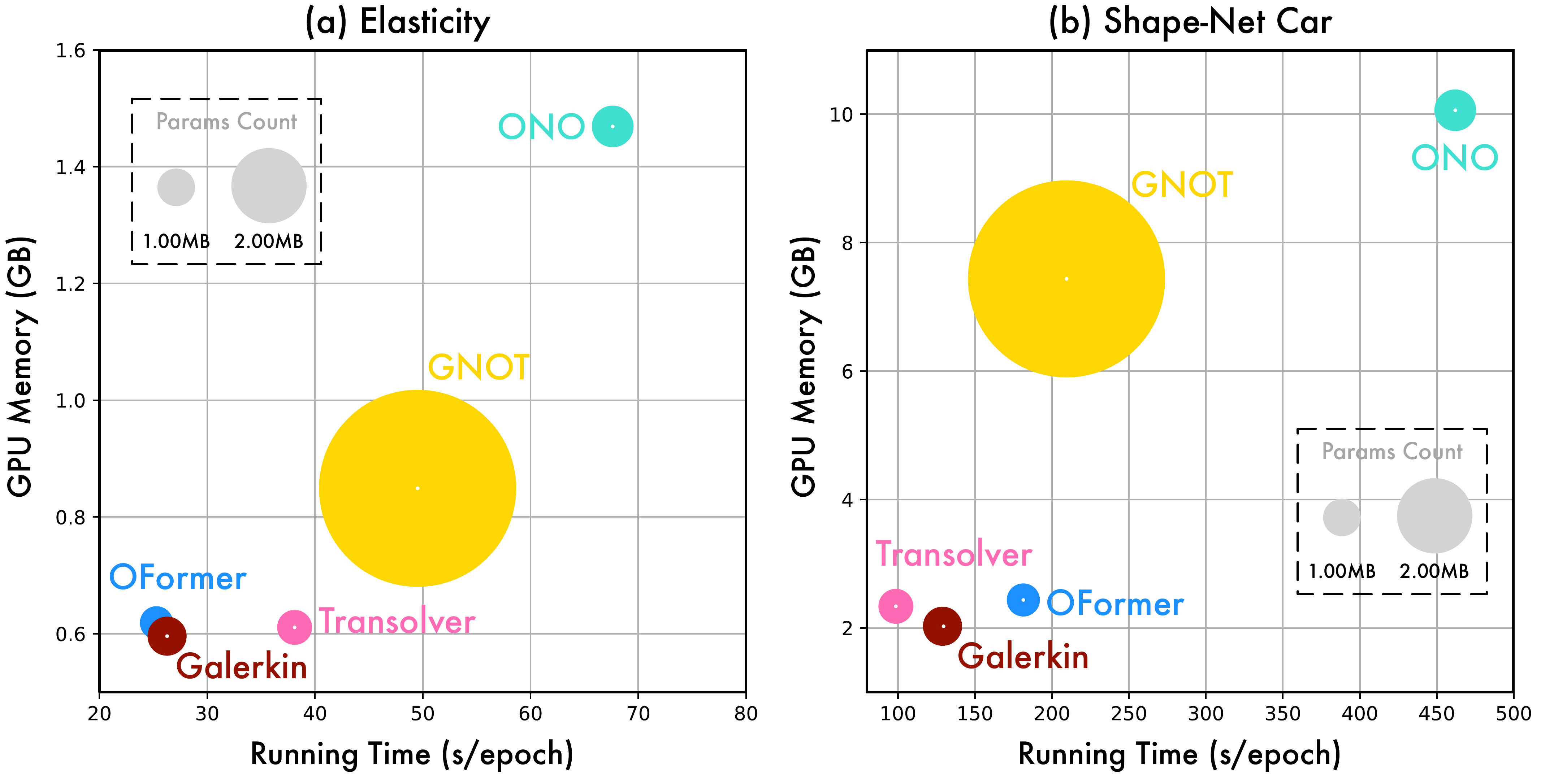}}
    \vspace{-5pt}
	\caption{Efficiency comparison on Elasticity (972 mesh points) and Shape-Net Car (32,186 mesh points). Metrics are measured on the batch size as 1 and one epoch contains 1000 iterations. The growth curves w.r.t.~input mesh size are provided in Appendix \ref{appdix:efficiency}.}
	\label{fig:efficiency}
\end{center}
\vspace{-25pt}
\end{figure}

\begin{figure*}[t]
\begin{center}
\centerline{\includegraphics[width=\textwidth]{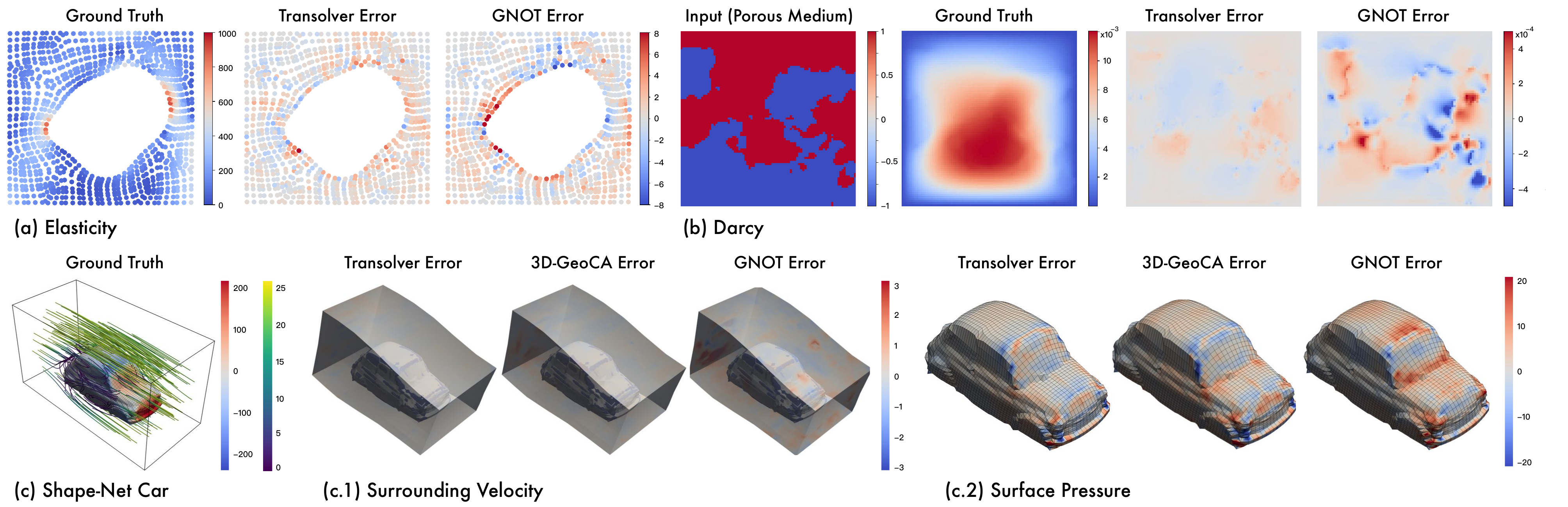}}
    \vspace{-5pt}
	\caption{Case study on error maps of different models. Notably, Shape-Net Car requires to predict the surrounding velocity and surface pressure simultaneously. For clearness, we plot the error on volume and surface separately. See Appendix \ref{appdix:full_vis_showcase} for more showcases.}
	\label{fig:case}
\end{center}
\vspace{-25pt}
\end{figure*}

\vspace{-5pt}
\paragraph{Efficiency} To further demonstrate the model practicability, we provide the model efficiency comparison in Figure~\ref{fig:efficiency}. We can find that in comparison with other Transformer-based models, Transolver presents favorable efficiency, considering running time, GPU memory and model parameters. Specifically, in Elasticity, Transolver achieves the 25.6\% error reduction (0.0064 vs.~0.0086) with 5x fewer parameters and 1.3x running speed than the second-best model GNOT. Especially for large-scale meshes, our design of linear-complexity Physics-Attention benefits more significantly, where Transolver surpasses Galerkin Transformer and OFormer in both running time and final performance.

\subsection{Model Analysis}

\paragraph{Physics-Attention analysis} Toward an intuitive understanding of Transolver, we visualize Physics-Attention in Figure \ref{fig:visual}. It is observed that slices can precisely capture mesh points under similar physics states. Also, it is worth noticing that we adopt the Softmax function in learning slice weights (Eq.~\eqref{equ:learn_slice}). Since Softmax function will make the learned weight distribution sharper, slices will also be optimized to capture more diverse patterns correspondingly, which empowers the model with better representation capability for intricate physics. {For example, in the inner-stress estimation task of Elasticity, we can find that slices learn to cover diverse subareas that are under similar extrusion degrees, such as the left and right of the hollow area, corners of the material, etc.} Further, learning physical states also frees our model from complex and unwieldy meshes. As shown in Figure \ref{fig:visual}(a), we randomly sample 50\% input mesh of Elasticity. Surprisingly, Transolver can still capture physical states precisely even for broken meshes.

\begin{table}[t]
    \vspace{-5pt}
	\caption{{Comparison of linear attention in Galerkin Transformer and Physics-Attention in Transolver. Kullback–Leibler (KL) divergence between learned attention weights and a uniform distribution is recorded, which is averaged from the whole test set.}}
	\label{tab:kl_distance}
	\vskip 0.1in
	\centering
	\begin{small}
		\begin{sc}
			\renewcommand{\multirowsetup}{\centering}
			\setlength{\tabcolsep}{2pt}
			\scalebox{1}{
			\begin{tabular}{l|c|c}
				\toprule
			\multirow{2}{*}{Benchmarks} & Galerkin & Transolver \\
                   & \scalebox{0.9}{\cite{Cao2021ChooseAT}} & \scalebox{0.9}{(Ours)} \\
			    \midrule
                  Elasticity (972 mesh points) & 0.3803 & \textbf{1.7795} \\
                  Darcy (7,225 mesh points) &	0.2739 & \textbf{1.8274} \\
				\bottomrule
			\end{tabular}}
		\end{sc}
	\end{small}
	\vspace{-5pt}
\end{table}

Besides, we also compare learned attention maps of the Transolver and Galerkin Transformer. As aforementioned, directly calculating attention among massive mesh points may overwhelm the model from learning informative relations \cite{wu2022flowformer}. {Statistical results in Table~\ref{tab:kl_distance} present that the linear attention among mesh points is closer to a degenerated uniform distribution than Physics-Attention.} {As a supplement, we also include the attention visualization in Figure~\ref{fig:visual}(b), where we can observe that the attention map among physics-aware tokens is much sharper than mesh-point attention. These results further verify the benefits of our physics-inspired design in facilitating attention learning.}

\begin{figure}[t]
\begin{center}
\centerline{\includegraphics[width=\columnwidth]{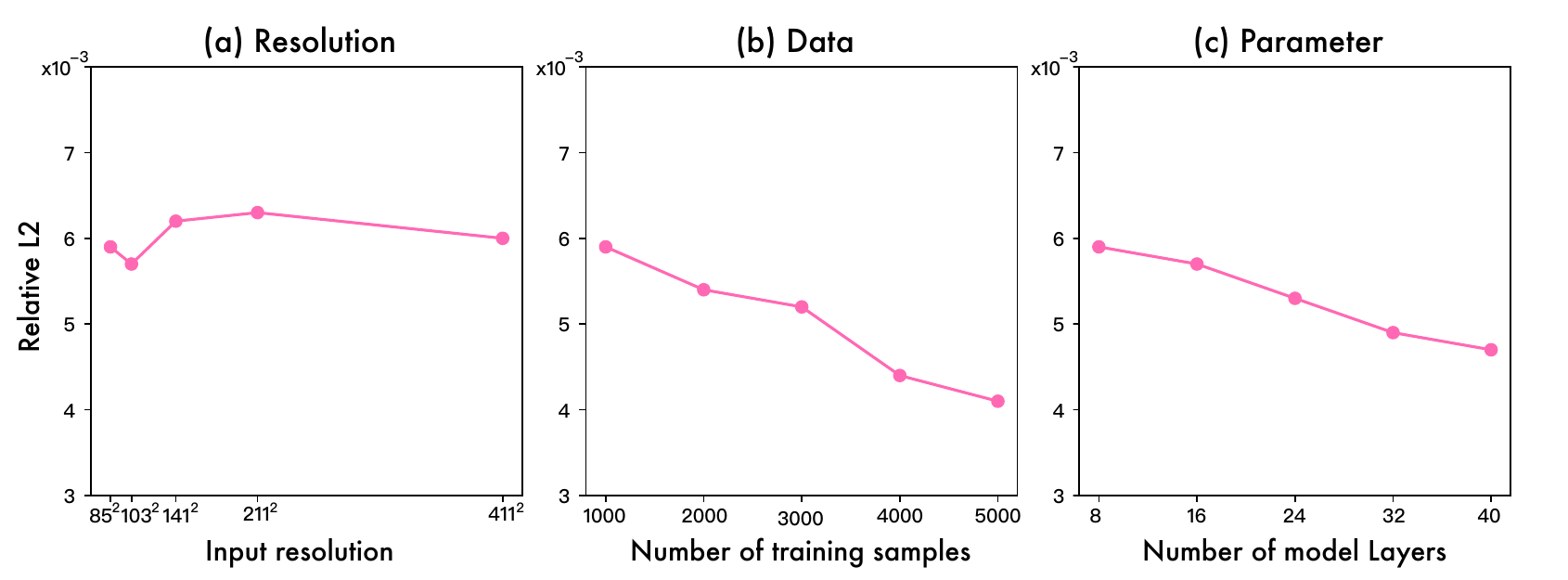}}
    \vspace{-5pt}
	\caption{Model scalability on Darcy. We re-train the Transolver for PDEs on different (a) resolutions, (b) training samples and (c) model layers. See Appendix \ref{appdix:scale} for full results on all benchmarks.}
	\label{fig:scalability}
\end{center}
\vspace{-25pt}
\end{figure}

\vspace{-5pt}
\paragraph{Case study} {We plot error maps of different models in Figure \ref{fig:case} to provide a clear comparison. It is observed that Transolver performs significantly better in boundaries and multiple material junctions}, such as the extrusion zone of elasticity, the medium boundary in Darcy and the curved region in the car surface. {Note that these areas usually involve intricate physical interactions, which require the model to precisely capture the geometry information and its hidden physics, thereby further verifying the effectiveness of our design in handling complex geometrics.} Especially, Transolver surpasses 3D-GeoCA \cite{anonymous2023geometryguided} and GNOT \cite{hao2023gnot} in predicting the wake flow behind the car and the pressure in the front area, demonstrating its capability in solving multiphysics PDEs on {hybrid} geometrics.

\vspace{-5pt}
\paragraph{PDE solving at scale} One well-acknowledged merit of Transformers is their scalability \cite{achiam2023gpt}. Thus, we also include a comprehensive test about the scalability of Transolver on resolution, data and model parameters in Figure \ref{fig:scalability}. Specifically, we gradually increase the PDE resolution to 25x of the original setting, training data and model parameters to 5x. It is observed that Transolver can achieve consistent performance at scaled resolutions and benefit from more training data and larger model parameters, posing a potential for a large-scale pre-trained PDE solver.

\vspace{-5pt}
\paragraph{Out-of-distribution (OOD) generalization} {In the previous research, neural solvers are mainly trained and tested with samples under the same or in-distribution PDE coefficients and varied initial or boundary conditions. For example, in the car design task, different samples of diverse shapes are all generated under the headwind with the same speed. As for the airfoil design, although training samples contain various Reynolds and angles of attacks, the test set is still under the same range of Reynolds and angles as the training set. Here, to further examine the generalizability of Transolver in real-world applications, we also experiment with OOD airfoil design tasks, where the test set has completely different Reynolds and angles of attacks.}

\begin{table}[t]
    \vspace{-5pt}
	\caption{{OOD generalization experiments on the AirfRANS. Relative error of lift coefficient ($C_L$) and Spearman’s rank correlations ($\rho_L$) are recorded. See Appendix \ref{appdix:ood} for complete results.}}
	\label{tab:ood}
	\vskip 0.1in
	\centering
	\begin{small}
		\begin{sc}
			\renewcommand{\multirowsetup}{\centering}
			\setlength{\tabcolsep}{2.5pt}
			\scalebox{1}{
			\begin{tabular}{l|cc|cc}
				\toprule
			\multirow{3}{*}{Models} & \multicolumn{2}{c|}{OOD Reynolds} & \multicolumn{2}{c}{OOD Angles}\\
   \cmidrule(lr){2-3}\cmidrule(lr){4-5}
   & $C_{L}$ $\downarrow$ & $\rho_{L}$ $\uparrow$ & $C_{L}$ $\downarrow$ & $\rho_{L}$ $\uparrow$ \\
			    \midrule
                    Simple MLP & 0.6205 & 0.9578 & 0.4128 & 0.9572 \\
                  GraphSAGE \citeyearpar{hamilton2017inductive} & 0.4333 & 0.9707 & \underline{0.2538} & 0.9894 \\
                  PointNet \citeyearpar{qi2017pointnet} & 0.3836 & 0.9806 & 0.4425 & 0.9784 \\
                  Graph U-Net \citeyearpar{gao2019graph} & 0.4664 & 0.9645 & 0.3756 & 0.9816 \\
                  \scalebox{0.95}{MeshGraphNet \citeyearpar{pfaff2021learning}} & \textcolor{gray}{1.7718} & \textcolor{gray}{0.7631} & \textcolor{gray}{0.6525} & \textcolor{gray}{0.8927}\\
                  \midrule
                  GNO \citeyearpar{li2020neural} & 0.4408 & \underline{0.9878} & 0.3038 & 0.9884\\
                  Galerkin \citeyearpar{Cao2021ChooseAT} & 0.4615 & 0.9826 & 0.3814 & 0.9821\\
                  GNOT \citeyearpar{hao2023gnot} & \underline{0.3268} & 0.9865 & 0.3497 & 0.9868 \\
                  GINO \citeyearpar{li2023geometryinformed} & 0.4180 & 0.9645 & 0.2583 & \underline{0.9923} \\
                  \midrule
                  \textbf{Transolver (Ours)} & \textbf{0.2996} & \textbf{0.9896} & \textbf{0.1500} & \textbf{0.9950} \\
				\bottomrule
			\end{tabular}}
		\end{sc}
	\end{small}
	\vspace{-15pt}
\end{table}

{As presented in Table \ref{tab:ood}, Transolver can handle OOD samples well, where it consistently performs best with Spearman's rank correlations of nearly 99\% on unseen Reynolds and angles of attacks. These results indicate that Transolver not only fits the training data but {also} captures some generalizable physical information, further highlighting the {advantage} of calculating attention among physical states.}

\section{Conclusions and Future Work}
This paper presents Transolver to solve PDEs on general geometrics. Unlike prior Transformer operators, Transolver proposes to apply attention to learned physical states, which empowers our model with endogenetic geometry-general capacity and benefits physical correlations modeling. Transolver not only performs impressively on well-established benchmarks but also excels in practical design tasks, which consist of extremely complex geometrics and tanglesome multiphysics interactions. Extensive analyses are provided to verify the model's performance, efficiency, scalability and out-of-distribution generalizability. {In the future, we will further explore the large-scale pre-training of Transolver in {pursuit of} foundation models for PDE solving.}

\section*{Acknowledgements}

This work was supported by the National Key Research and Development Plan (2021YFC3000905), the National Natural Science Foundation of China (U2342217 and 62022050), the BNRist Innovation Fund (BNR2024RC01010), and the National Engineering Research Center for Big Data Software. We thank our colleagues Hang Zhou and Yuezhou Ma for their suggestions in the OOD experiments of this paper.

\section*{Impact Statement}

This paper presents work whose goal is to advance the deep learning research for solving PDE. The proposed method is based on a new idea of capturing correlations among learned physical states, which could be inspiring for future research. And our model also performs well in {large-scale} design tasks, which can be useful in industrial production. Note that this paper mainly focuses on the scientific problem. When developing our approach, we are fully committed to ensuring ethical considerations are taken into account. Thus we believe there are no potential ethical risks in our work. 


\bibliography{example_paper}
\bibliographystyle{icml2024}

\newpage
\appendix
\onecolumn

\section{Proof of Theorem \ref{theorem:understanding}}\label{appendix:proof}

Firstly, we would like to present how to formalize canonical attention into a Monte-Carlo approximation of an integral operator as a Lemma, which is summarized from proofs provided by \citet{Cao2021ChooseAT} and \citet{jmlr_operator}.

\begin{lemma}\label{lemma:integral}
    The canonical attention mechanism in Transformers is a Monte-Carlo approximation of an integral operator.
\end{lemma}
\begin{proof} Given input function $\boldsymbol{u}:\Omega\to\mathbb{R}^{C}$, the integral operation $\mathcal{G}$ defined on the function space $\Omega\to\mathbb{R}^{C}$ is formalized as:
\begin{equation}
	\begin{split}\label{equ:integral_form}
 \mathcal{G}(\boldsymbol{u})(\mathbf{g}^\ast)=\int_{\Omega}\kappa(\mathbf{g}^\ast, \boldsymbol{\xi})\boldsymbol{u}(\boldsymbol{\xi})\mathrm{d}\boldsymbol{\xi},
    \end{split}
\end{equation}
where $\mathbf{g}^\ast\in\Omega\subset\mathbb{R}^{C_{\mathbf{g}}}$ and $\kappa(\cdot,\cdot)$ denotes the kernel function defined on $\Omega$. According to the formalization of attention, we propose to define the kernel function as follows:
\begin{equation}
	\begin{split}\label{equ:kernel}
\kappa(\mathbf{g}^\ast, \boldsymbol{\xi})= \left(\int_{\Omega}\exp \left(\left(\mathbf{W}_{\mathbf{q}}\boldsymbol{u}(\boldsymbol{\xi}^\prime)\right) \left(\mathbf{W}_{\mathbf{k}}\boldsymbol{u}(\boldsymbol{\xi})\right)^{\sf T}\right)\mathrm{d}\boldsymbol{\xi}^\prime\right)^{-1}\exp \left(\left(\mathbf{W}_{\mathbf{q}}\boldsymbol{u}(\mathbf{g}^\ast)\right) \left(\mathbf{W}_{\mathbf{k}}\boldsymbol{u}(\boldsymbol{\xi})\right)^{\sf T}\right)\mathbf{W}_{\mathbf{v}},
    \end{split}
\end{equation}
where $\mathbf{W}_{\mathbf{q}}, \mathbf{W}_{\mathbf{k}}, \mathbf{W}_{\mathbf{v}}\in\mathbb{R}^{C\times C}$. 

Suppose that there are $N$ discretized mesh points $\{\mathbf{g}_{1},\cdots,\mathbf{g}_{N}\}$, where $\mathbf{g}_{i}\in\Omega\subset\mathbb{R}^{C_{\mathbf{g}}}$. Approximating the inner-integral in Eq.~\eqref{equ:kernel} by Monte-Carlo, we have:
\begin{equation}
	\begin{split}
\int_{\Omega}\exp \left(\left(\mathbf{W}_{\mathbf{q}}\boldsymbol{u}(\boldsymbol{\xi}^\prime)\right) \left(\mathbf{W}_{\mathbf{k}}\boldsymbol{u}(\boldsymbol{\xi})\right)^{\sf T}\right)\mathrm{d}\boldsymbol{\xi}^\prime
 \approx \frac{|\Omega|}{N} \sum_{i=1}^{N}\exp \left(\left(\mathbf{W}_{\mathbf{q}}\boldsymbol{u}(\mathbf{g}_{i})\right) \left(\mathbf{W}_{\mathbf{k}}\boldsymbol{u}(\boldsymbol{\xi})\right)^{\sf T}\right).
    \end{split}
\end{equation}
Applying the above equation to Eq.~\eqref{equ:integral_form} and using the same approximation for the outer-integral, we have:
\begin{equation}
	\begin{split}\label{equ:final}
 \mathcal{G}(\boldsymbol{u})(\mathbf{g}^\ast)\approx\sum_{i=1}^{N}\frac{\exp \left(\left(\mathbf{W}_{\mathbf{q}}\boldsymbol{u}(\mathbf{g}^\ast)\right) \left(\mathbf{W}_{\mathbf{k}}\boldsymbol{u}(\mathbf{g}_{i})\right)^{\sf T}\right)\mathbf{W}_{\mathbf{v}}\boldsymbol{u}(\mathbf{g}_{i})}{ \sum_{j=1}^{N}\exp \left(\left(\mathbf{W}_{\mathbf{q}}\boldsymbol{u}(\mathbf{g}_{j})\right) \left(\mathbf{W}_{\mathbf{k}}\boldsymbol{u}(\mathbf{g}_{i})\right)^{\sf T}\right)},
    \end{split}
\end{equation}
which is the calculation of the attention mechanism with $\mathbf{W}_{\mathbf{q}}, \mathbf{W}_{\mathbf{k}}, \mathbf{W}_{\mathbf{q}}$ as linear layers for queries, keys and values.
\end{proof}

\begin{remark}[\textbf{Solving PDEs by learning integral neural operators}] \citet{li2020neural} firstly formalized the PDE-solving task as learning neural operators and defined the operator learning process as an iterative architecture, which corresponds to multiple layers in deep models. Then, \citet{li2021fourier} further defined each iteration step as a composition of a non-local integral operator and a local, nonlinear activation function. Since the nonlinear activation function can be easily parameterized by the feedforward layer. The key to solving PDEs is learning non-local integral operators. Thus, Lemma \ref{lemma:integral} proves that a Transformer model can theoretically work as a neural operator and be applied to solve PDEs.
\end{remark}

\begin{lemma}\label{lemma:projection}
    Suppose that $\Omega$ is a countable domain, the slice domain $\Omega_{\mathrm{s}}$ is isomorphic to $\Omega$.
\end{lemma}
\begin{proof}
    For the $i$-th element $\mathbf{g}_{i}\in\Omega$, we define its slice weight to the $j$-th slice $\mathbf{s}_j$ as $w_{\mathbf{g}_{i},\mathbf{s}_j}\in\mathbb{R}$. Given a constant $K\ge 1$ and $K\in\mathbb{N}$, since input domain $\Omega$ is countable, we can construct the projection $\boldsymbol{g}$ between input domain and slice domain $\Omega_{\mathrm{s}}$ as follows: for $i$ iterated from 1 to $+\infty$, its projected slice is defined as
    \begin{equation}
	\begin{split}\label{equ:projection}
\boldsymbol{g}(\mathbf{g}_{i})=\mathop{\arg\max}_{\mathbf{s}_j, \ \text{where}\ \left(\lfloor\frac{i-1}{K}\rfloor\times K\right) < j\leq \left(\left(\lfloor\frac{i-1}{K}\rfloor+1\right)\times K\right)}w_{\mathbf{g}_{i},\mathbf{s}_j}\ \ \text{subject\ to $j$-th slice is not projected before}.
    \end{split}
\end{equation}
Here we construct a bijection between the input domain and the slice domain. Thus, $\Omega\cong\Omega_{\mathrm{s}}$.
\end{proof}

Next, we will prove Theorem \ref{theorem:understanding}, which provides a theoretical understanding of our design in Physics-Attention. The formalization of Physics-Attention can be directly derived from integral based on proper assumptions.

\begin{proof}
According to Lemma \ref{lemma:projection}, we can obtain an isomorphic projection $\boldsymbol{g}$ between a countable input domain $\Omega$ and slice domain~$\Omega_{\mathrm{s}}$. Suppose that the slice weight $w_{\ast,\ast}:\overline\Omega\times\overline\Omega_{\mathrm{s}}\to\mathbb{R}$ is smooth in both $\overline\Omega$ and $\overline\Omega_{\mathrm{s}}$, where the $\overline\Omega$ and $\overline\Omega_{\mathrm{s}}$ denote the continuation of $\Omega$ and $\Omega_{\mathrm{s}}$ respectively, we can obtain $\boldsymbol{g}$ as a diffeomorphism projection.

Then, we define the value function $\boldsymbol{u}_{\mathrm{s}}$ on the physics-aware token domain $\Omega_{\mathrm{s}}$ as follows:
    \begin{equation}
	\begin{split}\label{equ:slice_function}
\boldsymbol{u}_\mathrm{s}(\boldsymbol{\xi}_{\mathrm{s}})=\left(\int_{\Omega}w_{\boldsymbol{\xi},\boldsymbol{\xi}_{\mathrm{s}}}\boldsymbol{u}(\boldsymbol{\xi})\mathrm{d}\boldsymbol{\xi}\right)/\left(\int_{\Omega}w_{\boldsymbol{\xi},\boldsymbol{\xi}_{\mathrm{s}}}\mathrm{d}\boldsymbol{\xi}\right),
    \end{split}
\end{equation}
which corresponds to the slice token definition in Eq.~\eqref{equ:token_encoding}.

Based on the above assumptions and definitions, we have:
\begin{equation}
	\begin{split}\label{equ:overall}
 & \mathcal{G}(\boldsymbol{u})(\mathbf{g}) = \int_{\Omega}\kappa(\mathbf{g}, \boldsymbol{\xi})\boldsymbol{u}(\boldsymbol{\xi})\mathrm{d}\boldsymbol{\xi} \\
 & = \int_{\Omega_{\text{s}}} \kappa_{\text{ms}}\big(\mathbf{g}, \boldsymbol{\xi}_{\text{s}}\big)\boldsymbol{u}_{\mathrm{s}}\left(\boldsymbol{\xi}_{\text{s}}\right)\mathrm{d}\boldsymbol{g}^{-1}(\boldsymbol{\xi}_{\text{s}})\qquad\qquad\qquad\qquad\qquad\qquad\qquad\ \ \text{($\kappa_{\mathrm{ms}}(\cdot,\cdot):\Omega\times\Omega_{\mathrm{s}}\to\mathbb{R}^{C\times C}$ is a kernel function)} \\
 & = \int_{\Omega_{\text{s}}} \kappa_{\text{ms}}\big(\mathbf{g}, \boldsymbol{\xi}_{\text{s}}\big)\boldsymbol{u}_{\mathrm{s}}\left(\boldsymbol{\xi}_{\text{s}}\right)|\det(\nabla_{\boldsymbol{\xi}_{\text{s}}}\boldsymbol{g}^{-1}(\boldsymbol{\xi}_{\text{s}}))|\mathrm{d}\boldsymbol{\xi}_{\text{s}} \\
  & = \int_{\Omega_{\text{s}}} \left(\frac{\int_{\Omega_{\text{s}}}w_{\mathbf{g}, \boldsymbol{\xi}_{\text{s}}^\prime}\kappa_{\text{ss}}\big(\boldsymbol{\xi}_{\text{s}}^\prime, \boldsymbol{\xi}_{\text{s}}\big)\mathrm{d}\boldsymbol{\xi}_{\text{s}}^\prime}{\int_{\Omega_{\text{s}}}w_{\mathbf{g}, \boldsymbol{\xi}_{\text{s}}^\prime}\mathrm{d}\boldsymbol{\xi}_{\text{s}}^\prime}\right)\boldsymbol{u}_{\mathrm{s}}\left(\boldsymbol{\xi}_{\text{s}}\right)|\det(\nabla_{\boldsymbol{\xi}_{\text{s}}}\boldsymbol{g}^{-1}(\boldsymbol{\xi}_{\text{s}}))|\mathrm{d}\boldsymbol{\xi}_{\text{s}} \qquad\ \text{($\kappa_{\mathrm{ms}}$ is a linear combination of $\kappa_{\mathrm{ss}}$ with weights $w_{\ast,\ast}$)} \\
  & = \int_{\Omega_{\text{s}}} \underbrace{w_{\mathbf{g}, \boldsymbol{\xi}_{\text{s}}^\prime}}_{\text{DeSlice}}\int_{\Omega_{\text{s}}}\underbrace{\kappa_{\text{ss}}\big(\boldsymbol{\xi}_{\text{s}}^\prime, \boldsymbol{\xi}_{\text{s}}\big)}_{\text{Attention among slice tokens}}\ \ \underbrace{\boldsymbol{u}_{\mathrm{s}}\left(\boldsymbol{\xi}_{\text{s}}\right)}_{\text{Slice token}}|\det(\nabla_{\boldsymbol{\xi}_{\text{s}}}\boldsymbol{g}^{-1}(\boldsymbol{\xi}_{\text{s}}))|\mathrm{d}\boldsymbol{\xi}_{\text{s}}\mathrm{d}\boldsymbol{\xi}_{\text{s}}^\prime \qquad\qquad\ \ \ \ \ \text{(Suppose that $\int_{\Omega_{\text{s}}}w_{\mathbf{g}, \boldsymbol{\xi}_{\text{s}}^\prime}\mathrm{d}\boldsymbol{\xi}_{\text{s}}^\prime=1$)} \\
  & \approx \underbrace{\sum_{j=1}^{M}\mathbf{w}_{i,j}}_{\text{Eq.~\eqref{equ:deslice}}}\underbrace{\sum_{t=1}^{M}\frac{\exp\left(\left(\mathbf{W}_{\mathbf{q}}\boldsymbol{u}_{\mathrm{s}}(\boldsymbol{\xi}_{\mathrm{s},j})\right)\left(\mathbf{W}_{\mathbf{k}}\boldsymbol{u}_{\mathrm{s}}(\boldsymbol{\xi}_{\mathrm{s},t})\right)^{\sf T}/\tau\right)}{\sum_{p=1}^{M}\exp\left(\left(\mathbf{W}_{\mathbf{q}}\boldsymbol{u}_{\mathrm{s}}(\boldsymbol{\xi}_{\mathrm{s},j})\right)\left(\mathbf{W}_{\mathbf{k}}\boldsymbol{u}_{\mathrm{s}}(\boldsymbol{\xi}_{\mathrm{s},p})\right)^{\sf T}/\tau\right)}\mathbf{W}_{\mathbf{v}}}_{\text{Eq.~\eqref{equ:attn}}}\Bigg(\underbrace{\frac{\sum_{p=1}^{N} \mathbf{w}_{p,t}\boldsymbol{u}(\mathbf{g}_{p})}{\sum_{p=1}^{N} \mathbf{w}_{p,t}}}_{\text{Eq.~\eqref{equ:token_encoding}}}\Bigg) \qquad\qquad\ \ \ \text{(Lemma \ref{lemma:integral})} \\
  & = \sum_{j=1}^{M}\mathbf{w}_{i,j}\sum_{t=1}^{M}\frac{\exp(\mathbf{q}_{j}\mathbf{k}_{t}^{\sf T}/\tau)}{\sum_{p=1}^{M}\exp(\mathbf{q}_{j}\mathbf{k}_{p}^{\sf T}/\tau)}\mathbf{v}_{t},
  \end{split}
\end{equation}

where $\kappa_{\mathrm{ms}}$ defines the kernel function between mesh points and slices, and $\kappa_{\mathrm{ss}}$ is defined among slices. Since slices are permutation-invariant in our implementation, we take $|\det(\nabla_{\boldsymbol{\xi}_{\text{s}}}\boldsymbol{g}^{-1}(\boldsymbol{\xi}_{\text{s}}))|=1$ for simplification. Different from the attention among mesh points, the usage of Lemma \ref{lemma:integral} here is based on the Monte-Carlo approximation in the slice domain.
\end{proof}


\section{Implementation Details} \label{appendix:detail}

In this section, we provide the details of our experiments, including \textbf{benchmarks}, \textbf{metrics}, and \textbf{implementations}.

\begin{table}[b]
\vspace{-15pt}
	\caption{Summary of experiment benchmarks, where the first six datasets are from FNO \cite{li2021fourier} and geo-FNO \cite{Li2022FourierNO}, Shape-Net Car is from \cite{umetani2018learning} and preprocessed by \cite{anonymous2023geometryguided}, and AirfRANS is from \cite{bonnet2022airfrans}. \#Mesh records the size of discretized meshes. \#Dataset is organized as the number of samples in training and test sets.}
	\label{tab:dataset_detail}
	\vskip 0.1in
	\centering
	\begin{small}
		\begin{sc}
			\renewcommand{\multirowsetup}{\centering}
			\setlength{\tabcolsep}{5.5pt}
			\scalebox{1}{
			\begin{tabular}{l|c|c|c|c|c|c}
				\toprule
			    Geometry & Benchmarks & \#Dim & \#Mesh & \#Input & \#Output & \#Dataset \\
			    \midrule
                 Point Cloud & Elasticity & 2D & 972 & Structure & Inner Stress & (1000, 200) \\
                 \midrule
			     Structured  & Plasticity & 2D+Time & 3,131 & External Force & Mesh Displacement & (900, 80) \\
        	Mesh & Airfoil & 2D & 11,271 & Structure & Mach Number & (1000, 200) \\
                  & Pipe & 2D & 16,641 & Structure & Fluid Velocity & (1000, 200) \\
                  \midrule
                  \multirow{2}{*}{Regular Grid}
		          & Navier–Stokes & 2D+Time & 4,096 & Past Velocity & Future Velocity & (1000, 200) \\	     
                    & Darcy & 2D & 7,225 & Porous Medium & Fluid Pressure & (1000, 200) \\
                \midrule
                Unstructured & Shape-Net Car & 3D & 32,186 & Structure & Velocity \& Pressure & (789, 100) \\
                Mesh & AirfRANS & 2D & 32,000 & Structure & Velocity \& Pressure & (800, 200) \\
				\bottomrule
			\end{tabular}}
		\end{sc}
	\end{small}
\end{table}

\vspace{-5pt}
\subsection{Benchmarks}
We extensively evaluate our model in eight benchmarks, whose information is summarized in Table \ref{tab:dataset_detail}. Note that these benchmarks involve the following three types of PDEs:
\begin{itemize}
    \item \textbf{Solid material} \cite{dym1973solid}: Elasticity and Plasticity.
    \item \textbf{Navier-Stokes equations for fluid} \cite{mclean2012continuum}: Airfoil, Pipe, Navier-Stokes, Shape-Net Car and AirfRANS.
    \item \textbf{Darcy’s law} \cite{hubbert1956darcy}: Darcy.
\end{itemize}
Here are the details of each benchmark.

\vspace{-5pt}
\paragraph{Elasticity} This benchmark is to estimate the inner stress of the elasticity material based on the material structure, which is discretized in 972 points \cite{Li2022FourierNO}. For each case, the input is a tensor in the shape of $972\times 2$, which contains the 2D position of each discretized point. The output is the stress of each point, thus in the shape of $972\times 1$. As for the experiment, 1000 samples with different structures are generated for training and another 200 samples are used for test.

\vspace{-5pt}
\paragraph{Plasticity} This benchmark is to predict the future deformation of the plasticity material under the impact from above by an arbitrary-shaped die \cite{Li2022FourierNO}. For each case, the input is the shape of the die, which is discretized into the structured mesh and recorded as a tensor with shape $101\times 31$. The output is the deformation of each mesh point in the future 20 time steps, that is a tensor in the shape of $20\times 101\times 31 \times 4$, which contains the deformation in four directions. Experimentally, 900 samples with different die shapes are used for model training and 80 new samples are for test.

\vspace{-5pt}
\paragraph{Airfoil} This task is to estimate the Mach number based on the airfoil shape, where the input shape is discretized into structured mesh with shape $221\times 51$ and the output is the Mach number for each mesh point \cite{Li2022FourierNO}. Here, all the shapes are deformed from the NACA-0012 case provided by the National Advisory Committee for Aeronautics. 1000 samples in different airfoil designs are used for training and the other 200 samples are for testing.

\vspace{-5pt}
\paragraph{Pipe} This benchmark is to estimate the horizontal fluid velocity based on the pipe structure \cite{Li2022FourierNO}. Each case discretizes the pipe into structured mesh with size $129\times 129$. Thus, for each case, the input tensor is in the shape of $129\times 129\times 2$, which contains the position of each discretized mesh point. The output is the velocity value for each point, thus in the shape of $129\times 129\times 1$. 1000 samples with different pipe shapes are used for model training and 200 new samples are for test, which are generated by controlling the centerline of the pipe.

\vspace{-5pt}
\paragraph{Navier-Stokes} This benchmark is to model the incompressible and viscous flow on a unit torus, where the fluid density is constant and viscosity is set as $10^{-5}$ \cite{li2021fourier}. The fluid field is discretized into $64\times 64$ regular grid. The task is to predict the fluid in the next 10 steps based on the observations in the past 10 steps. 1000 fluids with different initial conditions are generated for training, and 200 new samples are used for test.

\vspace{-5pt}
\paragraph{Darcy} This benchmark is to model the flow of fluid through a porous medium \cite{li2021fourier}. Experimentally, the process is discretized into a $421\times 421$ regular grid. Then we downsample the data into $85\times 85$ resolution for main experiments. The input of the model is the structure of the porous medium and the output is the fluid pressure for each grid. 1000 samples are used for training and 200 samples are generated for test, where different cases contain different medium structures.

\vspace{-5pt}
\paragraph{Shape-Net Car} This benchmark focuses on the drag coefficient estimation for the driving car, which is essential for car design. Overall, 889 samples with different car shapes are generated to simulate the 72 km/h speed driving situation \cite{umetani2018learning}, where the car shapes are from the ``car'' category of ShapeNet \cite{chang2015shapenet}. Concretely, they discretize the whole space into unstructured mesh with 32,186 mesh points and record the air around the car and the pressure over the surface. Here we follow the experiment setting in 3D-GeoCA \cite{anonymous2023geometryguided}, which takes 789 samples for training and the other 100 samples for testing. The input mesh of each sample is also preprocessed into the combination of mesh point position, signed distance function and normal vector. The model is trained to predict the velocity and pressure value for each point. Afterward, we can calculate the drag coefficient based on these estimated physics fields.

\vspace{-5pt}
\paragraph{AirfRANS} This dataset contains the high-fidelity simulation data for Reynolds-Averaged Navier–Stokes equations \cite{bonnet2022airfrans}, which is also used to assist airfoil design. Different from Airfoil \cite{Li2022FourierNO}, this benchmark involves more diverse airfoil shapes under finer discretized meshes. Specifically, it adopts airfoils in the 4 and 5 digits series of the National Advisory Committee for Aeronautics, which have been widely used historically. Each case is discretized into 32,000 mesh points. By changing the airfoil shape, Reynolds number, and angle of attack, AirfRANS provides 1000 samples, where 800 samples are used for training and 200 for the test set. Air velocity, pressure and viscosity are recorded for surrounding space and pressure is recorded for the surface. Note that both drag and lift coefficients can be calculated based on these physics quantities. However, as their original paper stated, air velocity is hard to estimate for airplanes, making all the deep models fail in drag coefficient estimation \cite{bonnet2022airfrans}. Thus, in the main text, we focus on the lift coefficient estimation and the pressure quantity on the volume and surface, which is essential to the take-off and landing stages of airplanes.

\subsection{Metrics}

Since our experiment consists of standard benchmarks and practical design tasks, we also include several design-oriented metrics in addition to the relative L2 for physics fields.

\vspace{-5pt}
\paragraph{Relative L2 for physics fields} Given the physics field $\mathbf{u}$ and the model predicted field $\widehat{\mathbf{u}}$, the relative L2 of model prediction can be calculated as follows:
\begin{equation}
	\begin{split}
\operatorname{Relative\ L2}=\frac{\|\mathbf{u}-\widehat{\mathbf{u}}\|}{\|\mathbf{u}\|}.
    \end{split}
\end{equation}
\vspace{-5pt}
\paragraph{Relative L2 for drag and lift coefficients} For Shape-Net Car and AirfRANS, we also calculated the drag and lift coefficients based on the estimated physics fields. For unit density fluid, the coefficient (drag or lift) is defined as follows:
\begin{equation}
	\begin{split}
C=\frac{2}{v^2A}\left(\int_{\partial\Omega}p(\boldsymbol{\xi})\left(\widehat{n}(\boldsymbol{\xi})\cdot\widehat{i}(\boldsymbol{\xi})\right)\mathrm{d}\boldsymbol{\xi} +\int_{\partial\Omega}\tau(\boldsymbol{\xi})\cdot\widehat{i}(\boldsymbol{\xi})\mathrm{d}\boldsymbol{\xi}\right),
    \end{split}
\end{equation}
where $v$ is the speed of the inlet flow, $A$ is the reference area, $\partial\Omega$ is the object surface, $p$ denotes the pressure function, $\widehat{n}$ means the outward unit normal vector of the surface, $\widehat{i}$ is the direction of the inlet flow and $\tau$ denotes wall shear stress on the surface. $\tau$ can be calculated from the air velocity near the surface \cite{mccormick1994aerodynamics}, which is usually much smaller than the pressure item. Specifically, for the drag coefficient of Shape-Net Car, $\widehat{i}$ is set as $(-1,0,0)$ and $A$ is the area of the smallest rectangle enclosing the front of cars. As for the lift coefficient of AirfRANS, $\widehat{i}$ is set as $(0,0,-1)$. The relative L2 is defined between the ground truth coefficient and the coefficient calculated from the predicted velocity and pressure field.

\vspace{-5pt}
\paragraph{Spearman’s rank correlations for drag and lift coefficients} Given $K$ samples in the test set with the ground truth coefficients $C=\{C^{1},\cdots, C^{K}\}$ (drag or lift) and the model predicted coefficients $\widehat{C}=\{\widehat{C}^{1},\cdots,\widehat{C}^{K}\}$, the Spearman correlation coefficient is defined as the Pearson correlation coefficient between the rank variables, that is:
\begin{equation}
	\begin{split}
\rho=\frac{\operatorname{cov}\left(R(C)R(\widehat{C})\right)}{\sigma_{R(C)}\sigma_{R(\widehat{C})}},
    \end{split}
\end{equation}
where $R$ is the ranking function, $\operatorname{cov}$ denotes the covariance and $\sigma$ represents the standard deviation of the rank variables. Thus, this metric is highly correlated to the model guide for design optimization. A higher correlation value indicates that it is easier to find the best design following the model-predicted coefficients \cite{spearman1961proof}.

\subsection{Implementations}

As shown in Table \ref{tab:training_model_detail}, all the baselines are trained and tested under the same training strategy. As for Transolver, the number of channels $C$ is set as 256 for high-dimensional inputs: Navier-Stokes, Shape-Net Car and AirfRANS, and 128 for the others. Further to balance efficiency, we set the number of slices $M$ as 32 for $C=256$ configuration and 64 for $C=128$. These configurations can also align our model parameters and running efficiency with other Transformer operators. Especially, we configure $\operatorname{Project}()$ in Eq.~\eqref{equ:learn_slice} as a single Linear layer for unstructured meshes: Elasticity, ShapeNet Car and AirfRANS, a convolution layer with $3\times 3$ kernel for others. Next, we will present the implementation of all the baselines. 

\vspace{-5pt}
\paragraph{Typical neural operators} All of these baselines have been widely examined in previous papers. Thus, for FNO and geo-FNO, we report their results following their official papers \cite{li2021fourier,Li2022FourierNO}. As for the other baselines, we follow the results in LSM \cite{wu2023LSM}. Note that all the other baselines expect geo-FNO cannot handle the unstructured mesh. Thus, their performances in Elasticity are evaluated by employing the special transformation in geo-FNO at the model beginning and ending layer. However, we find that geo-FNO degenerates seriously in practical design tasks (Shape-Net Car and AirfRANS) even with comprehensively searched hyperparameters, that is number of layers in $\{2,4,6,8\}$, number of channels in $\{128, 256, 512\}$ and number of Fourier basis in $\{16,32,64\}$. This may come from that geo-FNO is based on the periodic boundary assumption, while these two unstructured-mesh benchmarks apparently present complex boundaries. Since other neural operators are based on geo-FNO for unstructured meshes, we do not test them in practical design tasks.

\vspace{-5pt}
\paragraph{Transformer-based models} GNOT \cite{hao2023gnot} reports the performance of itself and OFormer \cite{li2023transformer}, Galerkin Transformer \cite{Cao2021ChooseAT} in part of six standard benchmarks, which also adopts a comprehensive hyperparameter search for themselves and these two baselines. Thus, we report their results based on their official paper and results from GNOT. As for the other untested benchmarks, we search the hyperparameters as follows: number of layers in $\{2,3,4,5,6,7,8\}$, number of channels in $\{128,256,512\}$, number to heads in $\{1,2,4,6,8\}$. As for HT-Net \cite{anonymous2023htnet} and FactFormer \cite{li2023scalable}, since they employ the square window and axial decomposition respectively, they are inapplicable to Elasticity. Thus, we evaluate it under the aforementioned hyperparameter search on the other baselines. As for the ONO \cite{anonymous2023improved} that is in submission, we adopted their official code provided in the OpenReview and rerun it under the same training and hyperparameter-search strategy as other baselines and also tried different linear attention designs \cite{Katharopoulos2020TransformersAR,kitaev2020reformer,Xiong2021NystrmformerAN,Cao2021ChooseAT,performer}.

Experimentally, we found that ONO \cite{anonymous2023improved} and OFormer \cite{li2023transformer} come across the unstable training problem on Shape-Net Car and AirfRANS. This may come from that ONO adopts the Cholesky decomposition to the channel attention map to ensure feature orthogonality, which requires the channel attention to be positive semidefinite. However, in large-scale mesh scenarios, this assumption may not be satisfied, making the model cannot be successfully executed. As for OFormer, it utilizes the cross-attention mechanism with target mesh points as queries and learned features as keys and values. However, directly calculating the attention among massive mesh points (over 32,000) is hard to optimize, causing the training loss curve to keep jittering. Thus, we do not report their performance in practical design tasks.

\vspace{-5pt}
\paragraph{Geometric deep models} We implement simple MLP, GraphSAGE \cite{hamilton2017inductive}, PointNet \cite{qi2017pointnet} and Graph U-Net \cite{gao2019graph} following the code base of AirfRANS \cite{bonnet2022airfrans}. As for GNO \cite{li2020neural} and 3D-GeoCA \cite{anonymous2023geometryguided}, we adopt the official code base of 3D-GeoCA. And we implement GINO based on its official code. Note that in the official paper, GINO is only trained to estimate the surface pressure of cars, which is not enough to calculate the drag coefficient. To ensure a comprehensive evaluation, we follow the experiment setting of 3D-GeoCA, which is to predict the surface pressure and surrounding air velocity simultaneously.

\begin{table}[t]
\vspace{-5pt}
	\caption{Training and model configurations of Transolver. Training configurations are directly from previous works without extra tuning \cite{bonnet2022airfrans,hao2023gnot,anonymous2023geometryguided}. Here $\mathcal{L}_{\mathrm{v}}$ and $\mathcal{L}_{\mathrm{s}}$ represent the loss on volume and surface fields respectively. As for Darcy, we adopt an additional spatial gradient regularization term $\mathcal{L}_{\mathrm{g}}$ following ONO \cite{anonymous2023improved}.}
	\label{tab:training_model_detail}
	\vskip 0.1in
	\centering
	\begin{small}
		\begin{sc}
			\renewcommand{\multirowsetup}{\centering}
			\setlength{\tabcolsep}{2.2pt}
			\begin{tabular}{l|ccccc|cccc}
				\toprule
                \multirow{3}{*}{Benchmarks} & \multicolumn{5}{c}{Training Configuration (Shared in all baselines)} & \multicolumn{4}{c}{Model Configuration} \\
                    \cmidrule(lr){2-6}\cmidrule(lr){7-10}
			  & Loss & Epochs & Initial LR & Optimizer & Batch Size & Layers $L$ & Heads & Channels $C$ &  Slices $M$ \\
			    \midrule
                 Elasticity & & \multirow{6}{*}{500} & \multirow{6}{*}{$10^{-3}$} & & 1 & \multirow{6}{*}{8} & \multirow{6}{*}{8} & 128 & 64 \\
			Plasticity & & & &  & 8 & & & 128 & 64 \\
        	Airfoil & Relative & & & AdamW & 4 & & & 128 & 64 \\
                Pipe & L2 & & & \citeyearpar{loshchilov2018decoupled} & 4 & & & 128 & 64 \\
                Navier–Stokes & & & & & 2 & & & 256 & 32 \\
                Darcy & $\mathcal{L}_{\mathrm{rL2}}+0.1\mathcal{L}_{\mathrm{g}}$ & & & & 4 &  & & 128 & 64 \\
                \midrule
                Shape-Net Car & $\mathcal{L}_{\mathrm{v}}+0.5\mathcal{L}_{\mathrm{s}}$ & 200 & \multirow{2}{*}{$10^{-3}$} & Adam & 1 & \multirow{2}{*}{8} & \multirow{2}{*}{8} & 256 & 32 \\
                AirfRANS & $\mathcal{L}_{\mathrm{v}}+\mathcal{L}_{\mathrm{s}}$  & 400 & & \citeyearpar{DBLP:journals/corr/KingmaB14} & 1 & & & 256 & 32  \\
				\bottomrule
			\end{tabular}
		\end{sc}
	\end{small}
\vspace{-5pt}
\end{table}

\section{Full Ablations}\label{appdix:ablations}

We include the complete ablations here as a supplement to Table \ref{tab:ablation} in the main text.
\vspace{-5pt}
\paragraph{Number of slices $M$} As demonstrated in Table \ref{tab:ablation_full}, increasing the number of slices $M$ will generally boost the model performance. Also, it is notable that increasing $M$ will also bring extra computation costs. {Too many slice tokens may also {lead} the attention calculation {to} potential noise or distraction, which could bring performance fluctuation or drop, such as in $M=1024$ cases of Elasticity and Navier-Stokes.} In this paper, we set $M$ as 64 for the models with 128 hidden channels and 32 for models with 256 hidden channels, which can accomplish a better balance between performance and efficiency.
\vspace{-5pt}
\paragraph{Learnable slices or regular squares} In all benchmarks, using fixed regular squares will damage the model performance, even in the Navier-Stokes and Darcy which are originally discretized in the regular grid. This may result from the intricate and spatially deformed physics states in PDEs, further demonstrating the advantage of learning adaptive slices in Transolver.

\paragraph{{Comparison with full attention}} {To highlight the advantages of Transolver, we also compare it with Plain Transformer~\cite{NIPS2017_3f5ee243}, which employs the full attention mechanism. Note that it is non-trivial to apply the plain Transformer to our benchmarks (1k-32k points). Even for the most powerful Large Language Model GPT-4 \cite{achiam2023gpt}, it can only handle up to 32k tokens, which is well-optimized and accelerated. When it comes to PDE-solving tasks, almost all of the previous Transformer-based models \cite{li2023transformer,hao2023gnot} utilize well-designed efficient attention for more than 1k mesh points instead of plain Transformer. Thus, the experiments of plain Transformer in over 1k tokens are with the help of the gradient-checkpointing technique \cite{chen2016training}, which can reduce the GPU memory but also significantly affect speed. Even so, we can only measure plain Transformer up to 7k tokens on one A100 40GB GPU.}

{Concretely, we downsample the Darcy dataset from 168,921 tokens ($411\times 411$) to 484 tokens ($22\times 22$) and replace Physics-Attention {by} canonical attention with identical architecture {elsewhere}. Note that directly downsampling the ground-truth data is unreasonable in the PDE context, since some physical interactions can only be observed in high resolution. That is why both models' performances drop seriously in 484 and 1,681 token settings. However, the relative promotions under different domain sizes are still referable. In Figure \ref{fig:scalability} of main text, we have demonstrated that Transolver performs stable in the range of 7,225 to 168,921 tokens. In the new experiments of Table \ref{tab:ablation_plain_trm}, we can find that in the smallest resolution, Transolver performs similarly to plain Transformer, while the advantages of Transolver are getting more significant {under} larger mesh points, demonstrating the superiority of Physics-Attention design in handling large-scale meshes.}

\begin{table}[t]
	\caption{Full ablations on Physics-Attention. We experiment on three variants: changing the number of slices $M$ and replacing the learnable slices with fixed regular squares in the shape of $4\times 4$ (\#Regular Squares). Efficiency is calculated on 1024 mesh points and batch size as 1. Since \#Regular Squares cannot handle unstructured inputs, we omit its efficiency and performance on Elasticity.}
	\label{tab:ablation_full}
	\vskip 0.1in
	\centering
	\begin{small}
		\begin{sc}
			\renewcommand{\multirowsetup}{\centering}
			\setlength{\tabcolsep}{4.5pt}
			\scalebox{1}{
			\begin{tabular}{l|c|cc|cccccc}
				\toprule
			\multicolumn{2}{c|}{\multirow{2}{*}{Ablations}} & {\#Memory} & {\#Time} & \multicolumn{6}{c}{{Relative L2}} \\
                \multicolumn{2}{c|}{} & {(GB)} & {(s/epoch)} & {Elasticity} & {Plasticity} & {Airfoil} & {Pipe} & {Naiver-Stokes} & {Darcy} \\
			    \midrule
                  & {1} & 0.60 & 37.76 & 0.0148 & 0.0140 & 0.0084 & 0.0087 & 0.1511 & 0.0386 \\	     		
                  & {8} & 0.60 & 37.82 & 0.0071 & 0.0028 & 0.0056 & 0.0040 & 0.1136 & 0.0096 \\	     
                  & {16} & 0.61 & 37.96 & 0.0067 & 0.0019 & 0.0057 & 0.0045 & 0.0958 & 0.0067 \\
                  & 32 & 0.62 & 38.00 & 0.0067 & 0.0015 & 0.0067 & 0.0042 & 0.0900 & 0.0063  \\
                  Number & 64 & 0.64 & 38.18 & 0.0064 & 0.0012 & 0.0053 & 0.0033 & 0.0871  & 0.0059 \\
                  of Slices & 96 & 0.68 & 38.31 & 0.0061 & \textbf{0.0008} & 0.0054 & 0.0033 & 0.0802 & 0.0055 \\
                  & 128 & 0.69 & 42.24 & 0.0058 & 0.0009 & 0.0049 & 0.0034 & \textbf{0.0783} & 0.0054 \\
                  & {256} & 0.81 & 39.13 & \textbf{0.0054} & 0.0013 & \textbf{0.0043} & \textbf{0.0032} & 0.0856 & \textbf{0.0050} \\
                  & {512} & 1.01 & 39.75 & 0.0059 & 0.0012 & 0.0045 & 0.0040 & 0.0914 & 0.0056 \\
                  & {1024} & 1.53 & 40.49 & 0.0068 & 0.0017 & 0.0048 & 0.0047 & 0.1003 & 0.0055 \\
                  \midrule
                  \multicolumn{2}{l|}{{Regular Squares}} & / & / & / & 0.0022 & 0.0071 & 0.0049 & 0.1077 & 0.0088   \\
				\bottomrule
			\end{tabular}}
		\end{sc}
	\end{small}
 \vspace{-10pt}
\end{table}

\begin{table}[t]
	\caption{{Comparison between Plain Transformer \cite{NIPS2017_3f5ee243} and Transolver at different resolutions of Darcy.}}
	\label{tab:ablation_plain_trm}
	\vskip 0.1in
	\centering
	\begin{small}
		\begin{sc}
			\renewcommand{\multirowsetup}{\centering}
			\setlength{\tabcolsep}{4pt}
			\scalebox{1}{
			\begin{tabular}{c|cccccccc}
				\toprule
			Number of Mesh Points & 484 & 1,681 & 3,364 & 7,225 & 10,609 & 19,881 & 44,521 & 168,921 \\
                (Resolutions) & (22$\times$22) & (41$\times$41) & (58$\times$58) & (85$\times$85) & (103$\times$103) & (141$\times$141) & (211$\times$211) & (411$\times$411) \\
			    \midrule
       Plain Transformer & 0.02017 & 0.0103 & 0.0073 & 0.0081 & OOM & OOM & OOM & OOM \\
       Transolver & 0.02019 & 0.0089 & 0.0058 & 0.0059 & 0.0057 & 0.0062 & 0.0063 & 0.0060 \\
       \midrule
       Relative Promotion & -0.1\% & 13.6\% & 20.5\% & 27.2\% & /& /& /& /\\
				\bottomrule
			\end{tabular}}
		\end{sc}
	\end{small}
 \vspace{-10pt}
\end{table}

\section{Addition Visualizations}\label{appdix:full_vis}

In this section, we provide more visualizations for learned slices and showcases as a supplement to Figure \ref{fig:visual} and Figure \ref{fig:case}.

\subsection{Learned Slices}\label{appdix:full_vis_slice}

\paragraph{Original mesh} We visualize the learned slices on 5 benchmarks: Shape-Net Car (Figure \ref{fig:slice_car}), Airfoil (Figure \ref{fig:slice_airfoil}), Pipe (Figure \ref{fig:slice_pipe}), Naiver-Stokes (Figure \ref{fig:slice_ns}), and Darcy (Figure \ref{fig:slice_darcy}). These visualizations provide valuable insights into the model's ability to capture diverse patterns and perceive subtle properties of physical states, including front-back-side pressure of driving cars, complex flow characteristics around the airfoil, fluid dynamics in pipes, swirling patterns of complex fluid interactions, and fluid-structure interactions along the porous medium. Especially, the model also learns a periodic diagonal pattern in the Naiver-Stokes equation, which corresponds to the periodic external force in the Navier-Stokes benchmark \cite{li2021fourier}. These findings demonstrate model's ability to learn underlying physical states behind complicated geometrics.

\vspace{-5pt}
\paragraph{Resampled mesh} To demonstrate that our design in learning physical states is free from concrete discretization, we also apply Transolver in resampled meshes in Figures \ref{fig:slice_airfoil_resample}-\ref{fig:slice_darcy_resample}, where we only keep 50\%-80\% mesh points of the original input. Note that this design may break the continuous and elaborately designed structure of the original mesh. Surprisingly, we can find that even for these broken meshes, Transolver can still capture physical states precisely, further verifying its geometry-general capability and highlighting the benefits of learning physical states.

\begin{figure*}[!htbp]
\begin{center}
\centerline{\includegraphics[width=\textwidth]{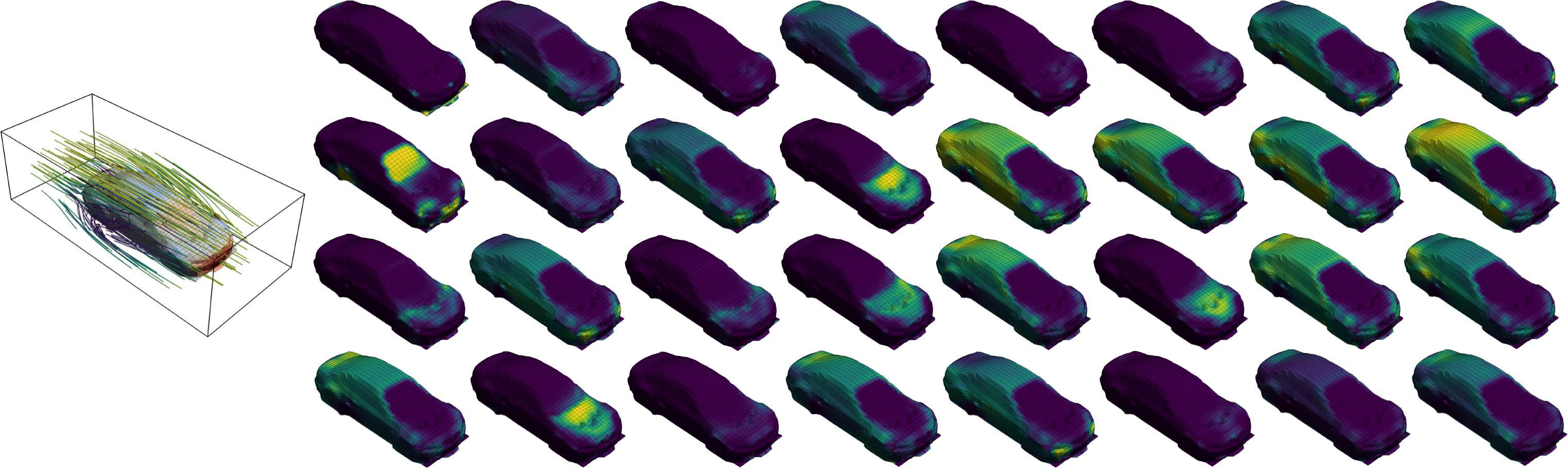}}
    \vspace{-5pt}
	\caption{Visualization of learned slices on the Shape-Net Car benchmark (number of slices $M=32$). Note that this benchmark involves both volume and mesh geometrics and velocity-pressure joint modeling. For clarity, we only present the slice weights on the surface here. Thus, the sum of visualized weights could be smaller than 1.}
	\label{fig:slice_car}
\end{center}
\vspace{-10pt}
\end{figure*}

\begin{figure*}[!htbp]
\begin{center}
\centerline{\includegraphics[width=\textwidth]{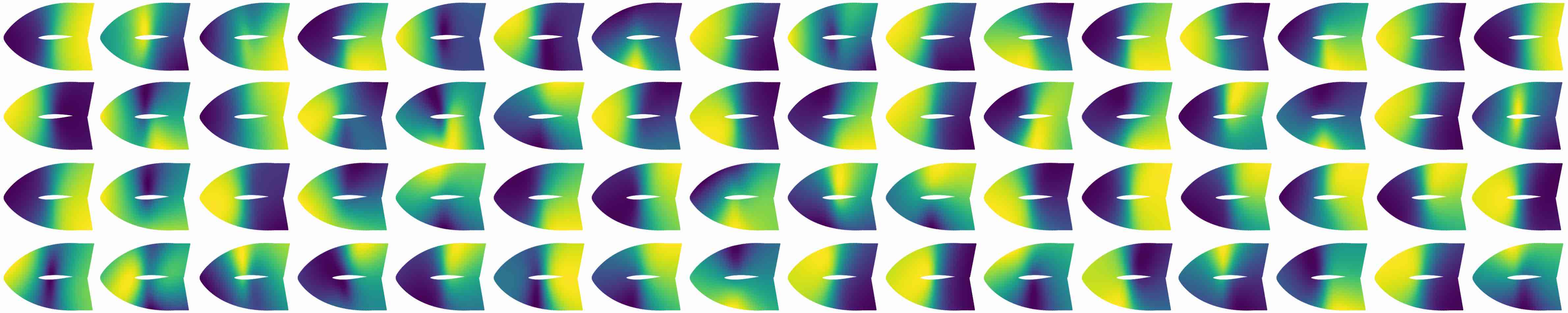}}
    \vspace{-5pt}
	\caption{Visualization of learned slices on the original Airfoil benchmark (number of slices $M=64$).}
	\label{fig:slice_airfoil}
\end{center}
\vspace{-10pt}
\end{figure*}

\begin{figure*}[!htbp]
\begin{center}
\centerline{\includegraphics[width=\textwidth]{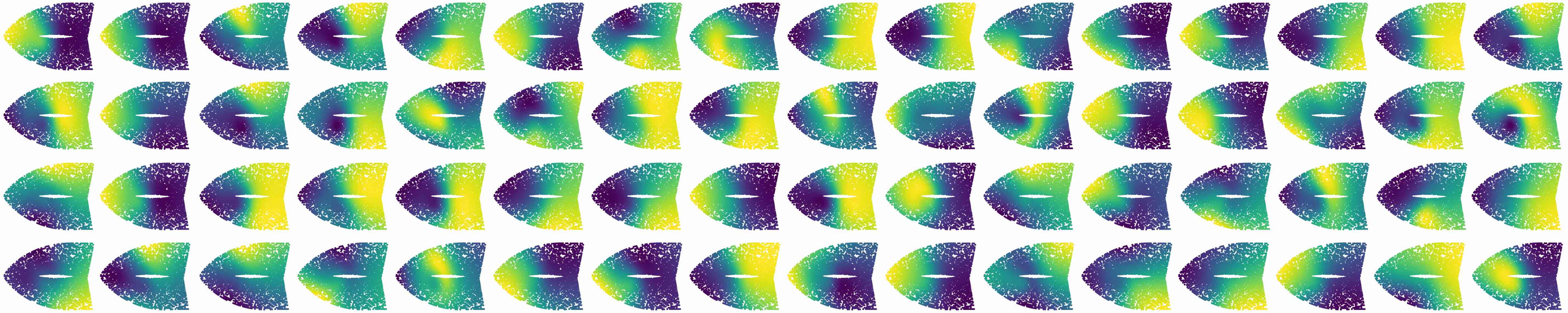}}
    \vspace{-5pt}
	\caption{Visualization of learned slices on the randomly resampled Airfoil benchmark (number of slices $M=64$).}
	\label{fig:slice_airfoil_resample}
\end{center}
\vspace{-10pt}
\end{figure*}

\begin{figure*}[!htbp]
\begin{center}
\centerline{\includegraphics[width=\textwidth]{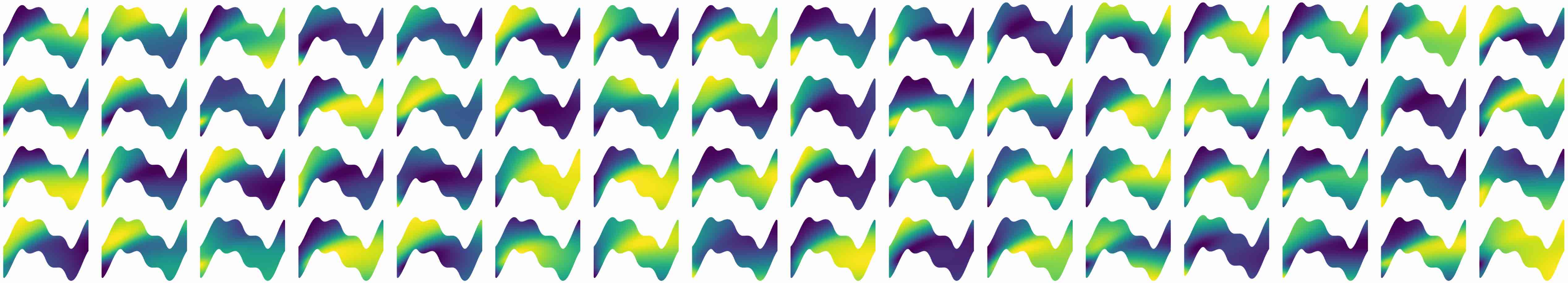}}
    \vspace{-5pt}
	\caption{Visualization of learned slices on the original Pipe benchmark (number of slices $M=64$).}
	\label{fig:slice_pipe}
\end{center}
\vspace{-10pt}
\end{figure*}

\begin{figure*}[!htbp]
\begin{center}
\centerline{\includegraphics[width=\textwidth]{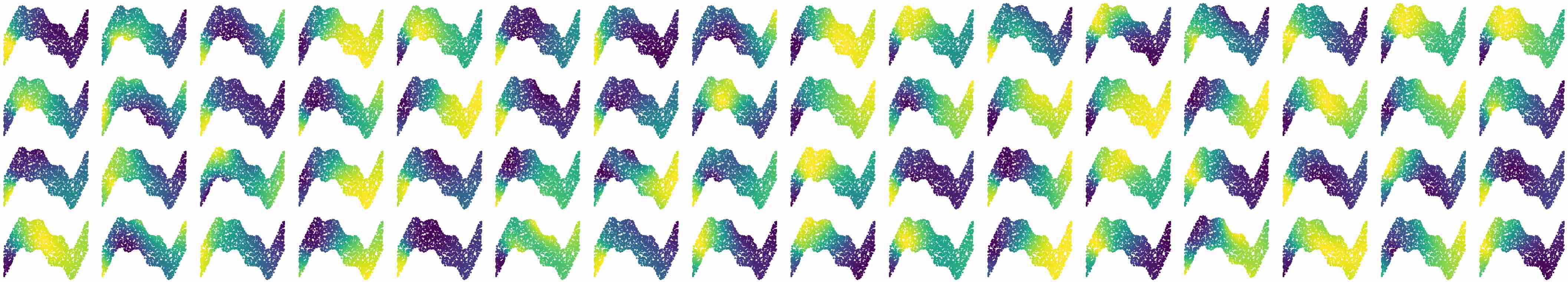}}
    \vspace{-5pt}
	\caption{Visualization of learned slices on the randomly resampled Pipe benchmark (number of slices $M=64$).}
	\label{fig:slice_pipe_resample}
\end{center}
\vspace{-10pt}
\end{figure*}

\begin{figure*}[!htbp]
\begin{center}
\centerline{\includegraphics[width=\textwidth]{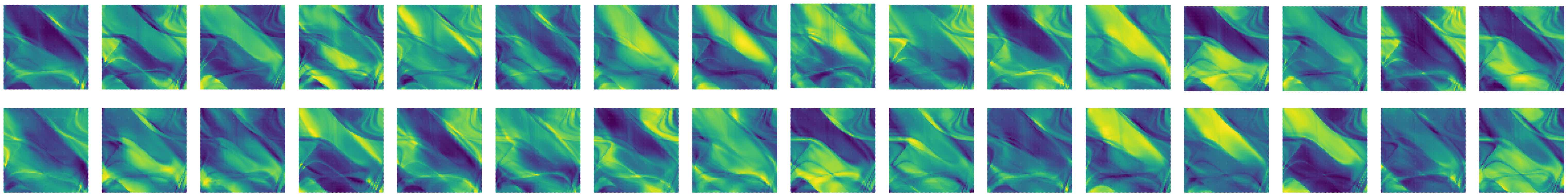}}
    \vspace{-5pt}
	\caption{Visualization of learned slices on the original Navier-Stokes benchmark (number of slices $M=32$).}
	\label{fig:slice_ns}
\end{center}
\vspace{-10pt}
\end{figure*}

\begin{figure*}[!htbp]
\begin{center}
\centerline{\includegraphics[width=\textwidth]{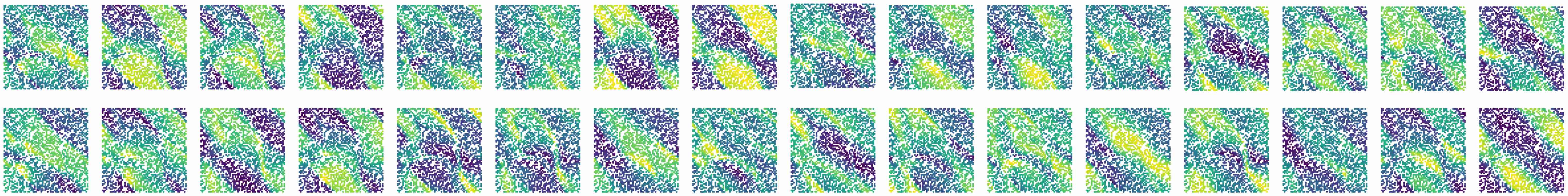}}
    \vspace{-5pt}
	\caption{Visualization of learned slices on the randomly resampled Navier-Stokes benchmark (number of slices $M=32$).}
	\label{fig:slice_ns_resample}
\end{center}
\vspace{-10pt}
\end{figure*}

\begin{figure*}[!htbp]
\begin{center}
\centerline{\includegraphics[width=\textwidth]{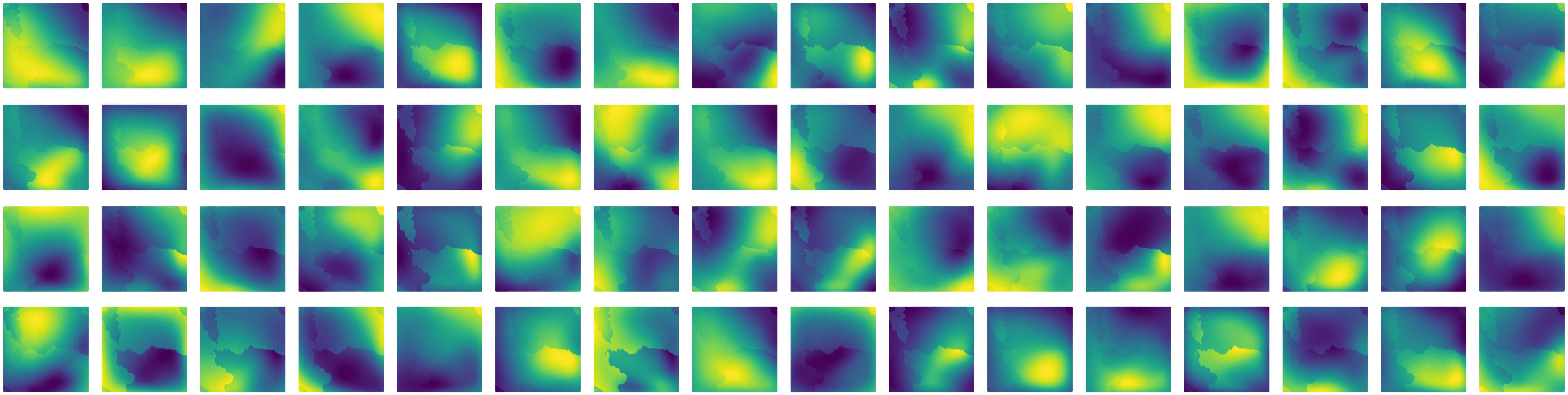}}
    \vspace{-5pt}
	\caption{Visualization of learned slices on the original Darcy benchmark (number of slices $M=64$).}
	\label{fig:slice_darcy}
\end{center}
\vspace{-10pt}
\end{figure*}

\begin{figure*}[!htbp]
\begin{center}
\centerline{\includegraphics[width=\textwidth]{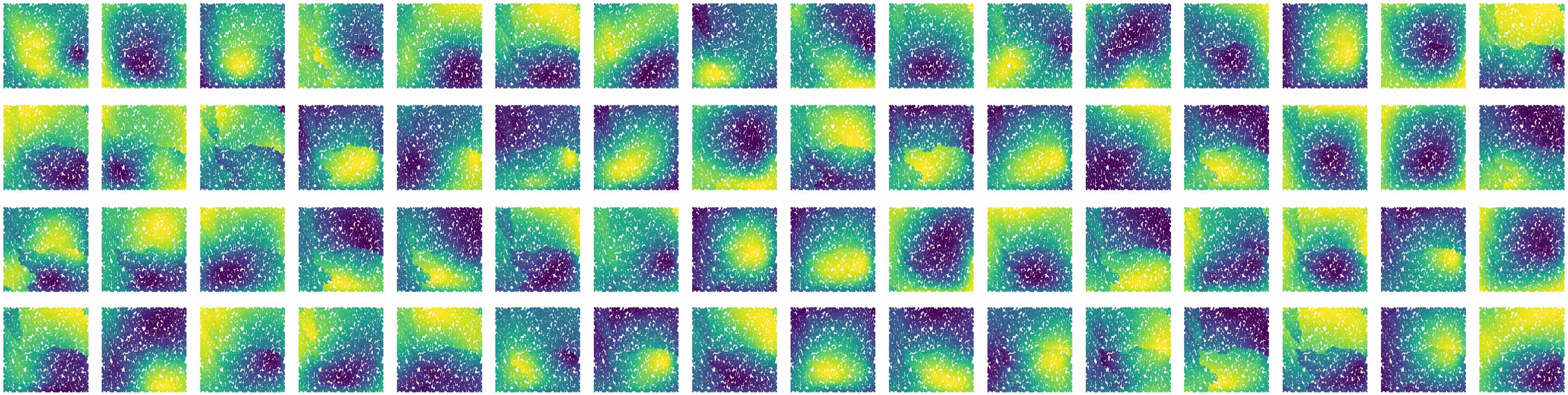}}
    \vspace{-5pt}
	\caption{Visualization of learned slices on the randomly resampled Darcy benchmark (number of slices $M=64$).}
	\label{fig:slice_darcy_resample}
\end{center}
\vspace{-15pt}
\end{figure*}

\subsection{Showcases}\label{appdix:full_vis_showcase}
As shown in Figure \ref{fig:showcase_all}, in comparison with the previous state-of-the-art models: LSM \cite{wu2023LSM} and GNOT \cite{hao2023gnot}, Transolver excels in capturing the deformation of Plasticity, the shock wave of Airfoil, fluid in the Pipe end and the swirling parts of Navier-Stokes. Here, GNOT fails in predicting the future deformation of Plasticity. 

\begin{figure*}[!htbp]
\begin{center}
\centerline{\includegraphics[width=\textwidth]{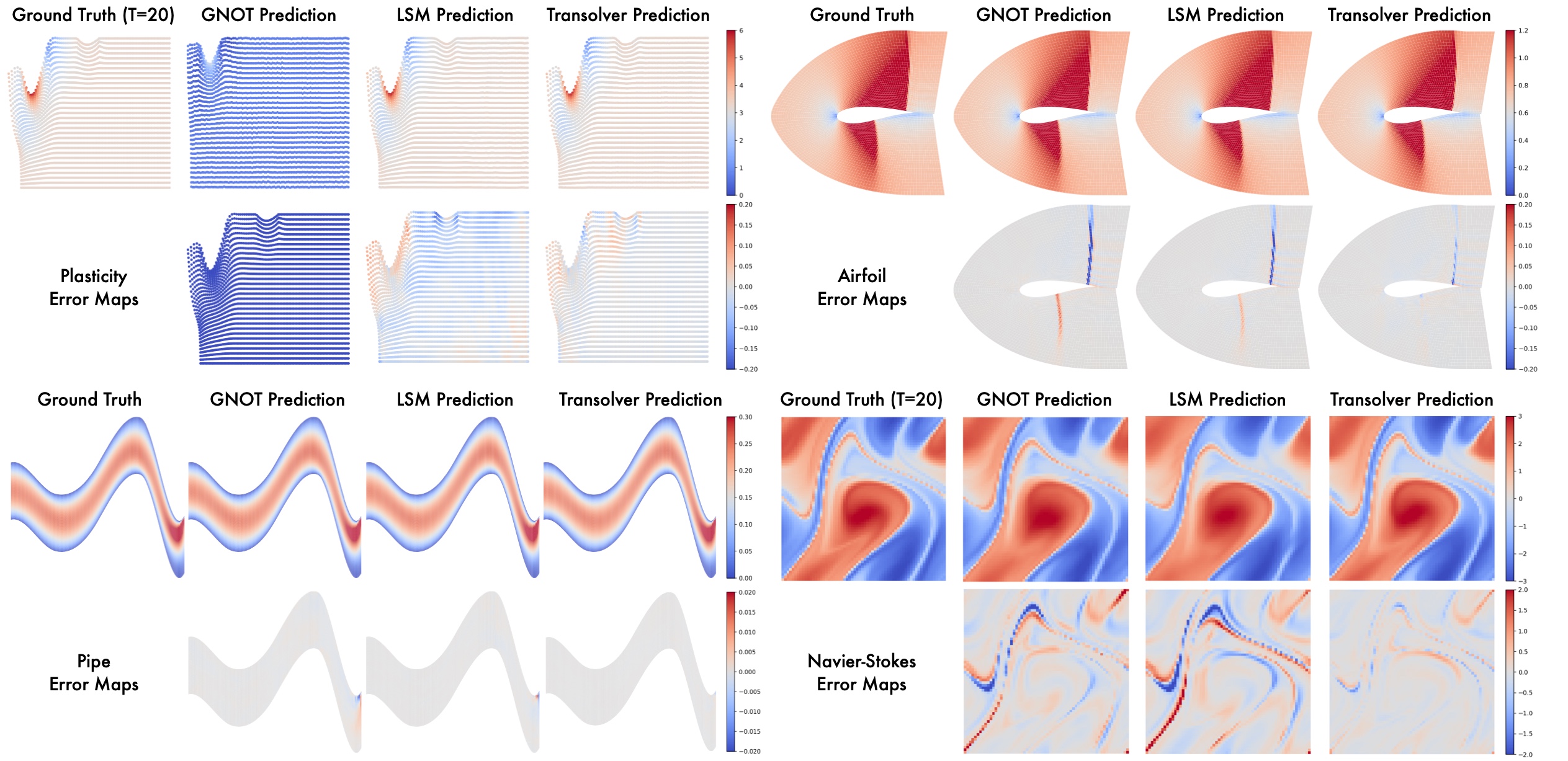}}
    \vspace{-5pt}
	\caption{Showcase comparison with the previous best models: LSM and GNOT. Both prediction results and error maps are provided.}
	\label{fig:showcase_all}
\end{center}
\vspace{-15pt}
\end{figure*}

\section{Addition Experiments}
Here we provide more experiments to complete the results of the main text and investigate new experiment settings.

\subsection{Model Scalability}\label{appdix:scale}
In the main text, we have investigated the model scalability in Darcy regarding resolution, data and parameters. Here we provide the scalability experiments on more benchmarks. Since we only have the data generation code of the Darcy benchmark, we only test the parameter scalability for the other datasets. As shown in Figure \ref{fig:scale_all}, most of the benchmarks will benefit from a larger model size, especially for Elasticity (Relative L2: 0.0064 for 8 layers, 0.0047 for 40 layers) and Airfoil (Relative L2: 0.0053 for 8 layers and 0.0037 for 40 layers). These results highlight the scaling potential of Transolver. However, for Plasticity (Relative L2: 0.0012) and Pipe (Relative L2: 0.0033) whose performances are already close to saturation, increasing the model parameter will bring performance fluctuation. The slight performance drop may also come from the stochasticity of optimization. Thus, a large dataset is expected to fully unlock the power of deep models in solving PDEs. Since we only focus on the model design in this paper, we would like to leave the data collection as future work.

\begin{figure*}[t]
\begin{center}
\centerline{\includegraphics[width=\textwidth]{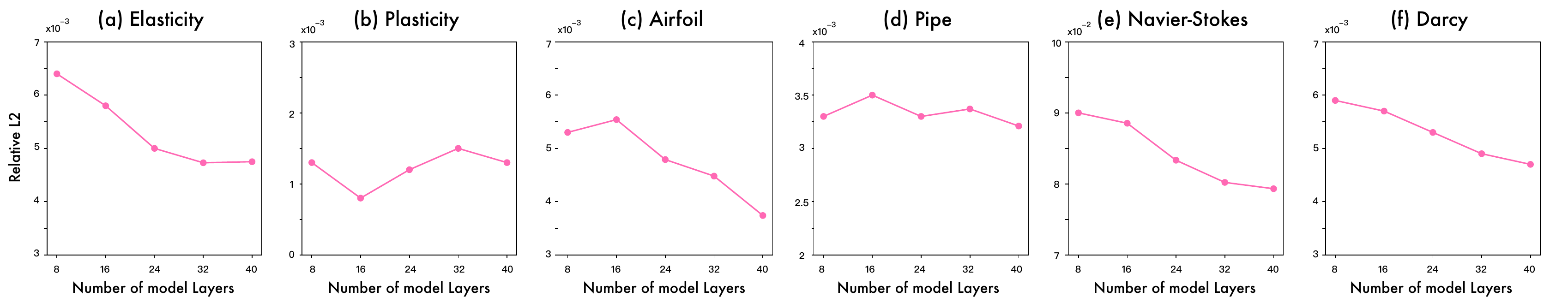}}
    \vspace{-5pt}
	\caption{Parameter scalability on six standard benchmarks, where we gradually increase the number of model layers from 8 to 40.}
	\label{fig:scale_all}
\end{center}
\vspace{-20pt}
\end{figure*}

\subsection{Adaptive Multiscale Modeling}\label{appdix:adaptive_multiscale}

Since $N$ mesh points are ascribed to $M$ slices as shown in Eq.~\eqref{equ:learn_slice}, it is easy to calculate that the expectation of the number of mesh points per slice is $\frac{N}{M}$. Thus, we can control the granularity of physical state modeling by adjusting $M$, where a larger $M$ will derive more slices, leading to more fine-grained physical states. Further, in deep models, we can configure $M$ as different values in different layers to conveniently achieve multiscale modeling.

As shown in Figure \ref{fig:adaptive_multiscale}, we further try the configuration that different layers are set with different numbers of slices. Since the calculation complexity of multiscale configuration is between $M=64$ and $M=32$, we also list the performance of these two official configurations for clarity in Table \ref{tab:analysis_multiscale}. Note that as we presented in Table \ref{tab:ablation_full}, in general, increasing the number of slices $M$ will boost the model performance. But, surprisingly, we find that in some cases, the multiscale configuration can perform comparably or even better than the $M=64$ official configuration. This may come from that both Darcy and Airfoil exhibit clear multiscale properties, thereby benefiting more from the hierarchical features.

\begin{table}[h]
\vspace{-5pt}
	\caption{Comparison between official and multiscale configuration. Efficiency is
calculated on 1024 mesh points and batch size as 1.}
	\label{tab:analysis_multiscale}
	\vskip 0.1in
	\centering
	\begin{small}
		\begin{sc}
			\renewcommand{\multirowsetup}{\centering}
			\setlength{\tabcolsep}{3.5pt}
			\scalebox{1}{
			\begin{tabular}{l|cc|cccccc}
				\toprule
			\multirow{2}{*}{Designs} & {\#Mem} & {\#Time} & \multicolumn{6}{c}{{Relative L2}} \\
                 & {(GB)} & {(s/epoch)} & {Elasticity} & {Plasticity} & {Airfoil} & {Pipe} & {Naiver-Stokes} & {Darcy} \\
			    \midrule
			     Official Config $M=32$ & 0.62 & 38.00 & 0.0067 & 0.0015 & 0.0067 & 0.0042 & 0.0900 & 0.0063  \\
                  Official Config $M=64$ & 0.64 & 38.18 & \textbf{0.0064} & \textbf{0.0012} & 0.0053 & \textbf{0.0033} & \textbf{0.0871}  & 0.0059 \\
                  \midrule
                  Multiscale Config & 0.62 & 38.10 & 0.0066 & \textbf{0.0012} & \textbf{0.0050} & 0.0037 & 0.0891 & \textbf{0.0056} \\
                  \bottomrule
			\end{tabular}}
		\end{sc}
	\end{small}
\end{table}

Different from the well-acknowledged multiscale deep models, such as Swin Transformer \cite{liu2021Swin} or U-Net~\cite{ronneberger2015u}, our design in learning slices is not limited by the discretization. This means that we can set the number of slices $M$ at will, regardless of the exact division restriction. This also makes our model free from inflexible padding operations. In this paper, we mainly experiment with the official configuration. We would like to leave the exploration of the multiscale architecture of Transolver as a future work.

\begin{figure*}[b]
\vspace{-10pt}
\begin{center}
\centerline{\includegraphics[width=0.6\textwidth]{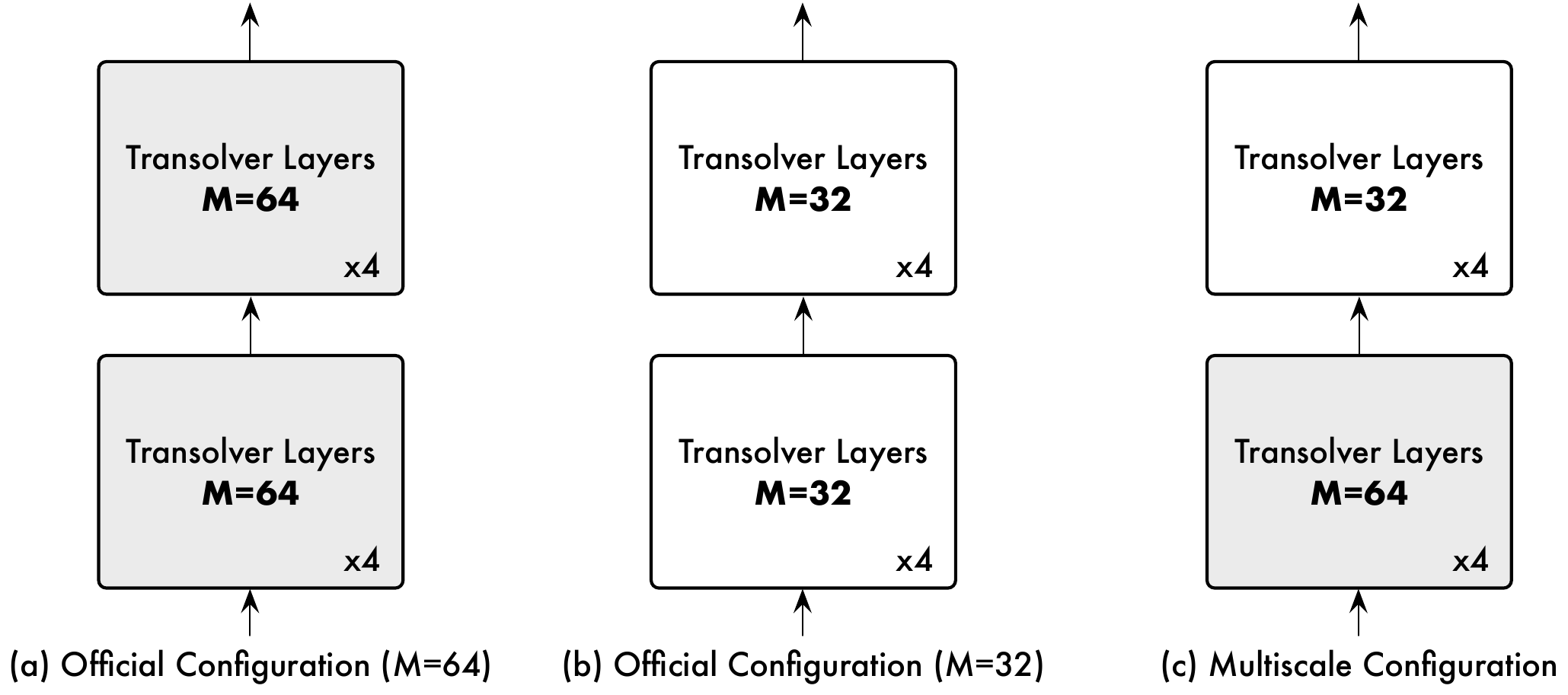}}
    \vspace{-5pt}
	\caption{Illustration of adaptive multiscale modeling in Transolver. In addition to the official configuration that all Transolver layers share the same number of slices, we also provide the multiscale configuration where the last 4 layers only use half slices w.r.t~first 4 layers.}
	\label{fig:adaptive_multiscale}
\end{center}
\vspace{-15pt}
\end{figure*}


\subsection{{OOD Generalization}}\label{appdix:ood}

{In Table \ref{tab:ood}, we have provided experiments of out-of-distribution Reynolds and angles of attacks. As a supplement, we include the detailed settings in Table \ref{tab:ood_settings}, where the training and test sets contain completely different Reynolds or angles of attacks.}

{We also include the complete results in Table \ref{tab:mainres_design_ood_full}. It is impressive that Transolver can still achieve the consistent best performance in the out-of-distribution settings. Specifically, the relative promotion of Transolver is more significant than the i.i.d.~setting, further demonstrating Transolver's generalizability advantage. This may come from our special design in learning physical states, which enables Transolver to capture more foundational physics information.}

\begin{table}[t]
    \vspace{-5pt}
	\caption{{Settings of OOD generalization experiments on the AirfRANS. The range of Reynolds and angles are listed.}}
	\label{tab:ood_settings}
	\vskip 0.1in
	\centering
	\begin{small}
		\begin{sc}
			\renewcommand{\multirowsetup}{\centering}
			\setlength{\tabcolsep}{10pt}
			\scalebox{1}{
			\begin{tabular}{l|cc|cc}
				\toprule
		\multirow{3}{*}{Dataset} & \multicolumn{2}{c|}{OOD Reynolds} & \multicolumn{2}{c}{OOD Angles}\\
   \cmidrule(lr){2-3}\cmidrule(lr){4-5}
   & Reynolds Range & Samples & Angles Range & Samples \\
			    \midrule
                    Training set & $[3\times 10^6, 5\times 10^6]$ & 500 & $[-2.5^\circ, 12.5^\circ]$  & 800 \\
                    Test set & $[2\times 10^6, 3\times 10^6]\cup [5\times 10^6, 6\times 10^6]$ & 500 & $[-5^\circ, -2.5^\circ]\cup[12.5^\circ, 15^\circ]$ & 200 \\
				\bottomrule
			\end{tabular}}
		\end{sc}
	\end{small}
	\vspace{-5pt}
\end{table}

\begin{table*}[h]
\vspace{-10pt}
	\caption{Performance comparison on out-of-distribution tasks of AirfRANS. Relative L2 of the surrounding (Volume) and surface (Surf) physics fields, the relative L2 of lift coefficient ($C_{L}$) is also recorded, along with Spearman’s rank correlations $\rho_{L}$.}
	\label{tab:mainres_design_ood_full}
	\vspace{-5pt}
	\vskip 0.15in
	\centering
	\begin{small}
		\begin{sc}
			\renewcommand{\multirowsetup}{\centering}
			\setlength{\tabcolsep}{4.5pt}
			\begin{tabular}{l|cccccccccc}
				\toprule
                    \multirow{3}{*}{Model} & \multicolumn{4}{c}{OOD Reynolds} & \multicolumn{4}{c}{OOD Angles} \\
                    \cmidrule(lr){2-5}\cmidrule(lr){6-9}
				& Volume $\downarrow$  & Surf $\downarrow$ & $C_{D}$ $\downarrow$ & $\rho_{D}$ $\uparrow$ & Volume $\downarrow$ & Surf $\downarrow$ & $C_{L}$ $\downarrow$ & $\rho_{L}$ $\uparrow$  \\
                    \midrule
                    Simple MLP & 0.0669 & 0.1153 & 0.6205 & 0.9578 & 0.1309 & 0.3311 & 0.4128 & 0.9572 \\
                    GraphSAGE \citep{hamilton2017inductive} & 0.0798 & 0.1254 & 0.4333 & 0.9707 & 0.1192 & \underline{0.2359} & \underline{0.2538} & 0.9894 \\
                    PointNet \citep{qi2017pointnet} & 0.0838 & 0.1403 & 0.3836 & 0.9806 & 0.2021 & 0.4649 & 0.4425 & 0.9784 \\
                    Graph U-Net \citep{gao2019graph} & 0.0538 & 0.1168 & 0.4664 & 0.9645 & 0.0979 & 0.2391 & 0.3756 & 0.9816 \\
                    MeshGraphNet \citep{pfaff2021learning} & \textcolor{gray}{0.2789} & \textcolor{gray}{0.2382} & \textcolor{gray}{1.7718} & \textcolor{gray}{0.7631} & \textcolor{gray}{0.4902} & \textcolor{gray}{1.1071} & \textcolor{gray}{0.6525} & \textcolor{gray}{0.8927}\\
                    \midrule
                    GNO \citep{li2020neural} & 0.0833 & 0.1562 & 0.4408 & \underline{0.9878} & 0.1626 & \underline{0.2359} & 0.3038 & 0.9884  \\
                    Galerkin \citep{Cao2021ChooseAT} & 0.0330 & 0.0972 & 0.4615 & 0.9826 & 0.0577 & 0.2773 & 0.3814 & 0.9821\\
                    GNOT \citep{hao2023gnot} & \underline{0.0305} & \underline{0.0959} & \underline{0.3268} & 0.9865 & \underline{0.0471} & 0.3466 & 0.3497 & 0.9868 \\
                    GINO \citep{li2023geometryinformed} & 0.0839 & 0.1825 & 0.4180 & 0.9645 & 0.1589 & 0.2469 & 0.2583 & \underline{0.9923} \\
                    \midrule
                    \textbf{Transolver (Ours)} & \textbf{0.0143} & \textbf{0.0364} & \textbf{0.2996} & \textbf{0.9896} & \textbf{0.0357} & \textbf{0.2275} & \textbf{0.1500} & \textbf{0.9950} \\
				\bottomrule
			\end{tabular}
		\end{sc}
	\end{small}
    \vspace{-5pt}
\end{table*}

\subsection{{Apply to Lagrangian Settings}}
{In the main text, we follow the convention of previous neural operators \cite{li2021fourier,wu2023LSM} and experiment with Eulerian datasets, where the geometry of input data is fixed. There is another branch of tasks, named Lagrangian settings, which simulates the dynamics system (e.g.~fluid) by tracking a series of particles. To further verify the effectiveness of Transolver in handling ever-changing geometrics, we also experiment with a Lagrangian PDE-solving task.}

{As for Lagrangian settings, the convention is to construct a graph at each timestep and utilize GNN-based models to capture local interactions among particles \cite{sanchez2020learning}. Although the capability to process irregular meshes also enables Transolver to receive the scattered particles as inputs, the essential graph information is missing. Thus, we did a preliminary experiment on the WaterDrop process \cite{sanchez2020learning}, whose task is to predict the future 994 steps based on the past 6 steps for 2,000 particles. Concretely, we enhance the GNN message passing with an additional Transolver layer in parallel. Both models are trained with 1M iterations, which requires 2 days in one A100 40GB GPU.}

{As shown in Table \ref{tab:lagrange}, Transolver can further boost the GNN performance {by a sharp margin} and generate a more accurate future, especially for the water splash process. This result verifies the benefits of Transolver in enhancing physics learning.}

\begin{table*}[h]
\vspace{-5pt}
    \caption{Performance comparison on the Lagrangian WaterDrop dataset. Position MSE of 2,000 predicted particles is recorded.}\label{tab:lagrange}
    \vspace{-5pt}
    \setlength{\tabcolsep}{3pt}
    \begin{minipage}[!b]{0.45\textwidth}
    \begin{table}[H]
    \begin{center}
    \begin{threeparttable}
    \begin{small}
    \begin{sc}
    \begin{tabular}{l|ccc}
    \toprule
    Model & MSE $\downarrow$  \\
    \midrule
    GNN \citep{sanchez2020learning} & 0.0182 \\
    \textbf{GNN + Transolver (Ours)} & \textbf{0.0069} \\
    \midrule
    Relative Promotion & 62.1\%  \\ 
    \bottomrule
    \end{tabular}
    \end{sc}
    \end{small}
    \end{threeparttable}
    \end{center}
    \end{table}
    \end{minipage}
    \hfill
    \begin{minipage}[!b]{0.53\textwidth}
      \begin{center}
    \vspace{10pt}
    \hspace{-5pt}
    \includegraphics[width=0.99\textwidth]{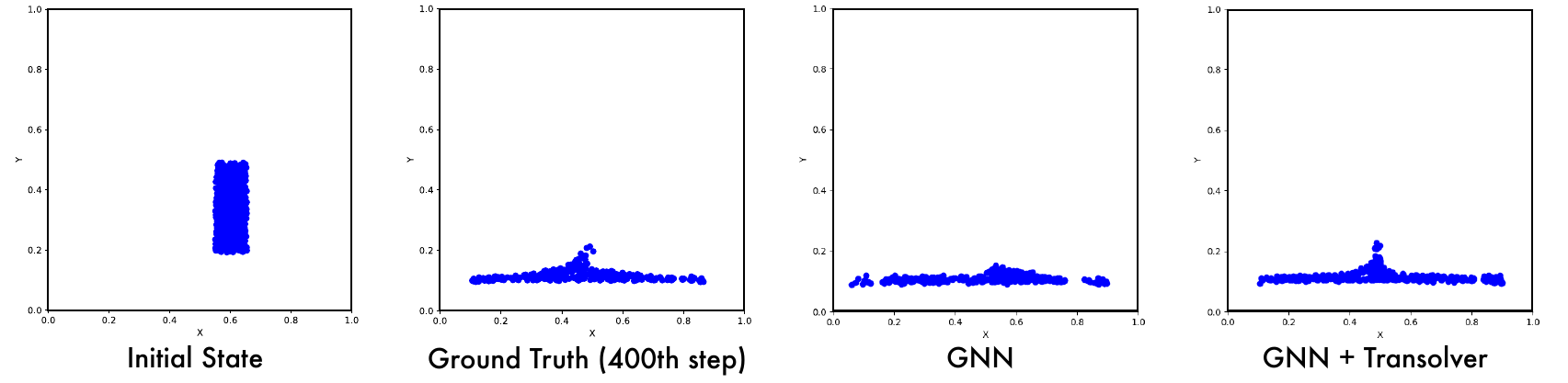}
  \end{center}
    \end{minipage}
    \label{nsresults}
    \vspace{-10pt}
\end{table*}

\subsection{Standard Deviations}
We repeat all the experiments three times and provide standard deviations here. As shown in Table \ref{tab:mainres_standard_std}, Transolver surpasses the previous state-of-the-art models with high confidence. It is worth noticing that we compare Transolver with the second-best model on each benchmark. It is hard to outperform all the previous models consistently given that the previous models have shown ups and downs on different benchmarks, further verifying the effectiveness of our model. 

Especially for AirfRANS \cite{bonnet2022airfrans} that contains diverse conditions on airfoil shape, Reynolds number and angle of attack, according to their official paper, they ``choose to only run 1000 simulations as one of the goals of this dataset is to be close to real-world settings, i.e. limited quantity of data.'' Thus, this limited data setting will result in a relatively large deviation. Notably, even in this hard setting, Transolver still surpasses the second-best model with 95\% confidence.

\begin{table*}[h]
\vspace{-5pt}
	\caption{Standard deviations of Transolver on all experiments. For clarity, we also list the performance of the second-best model. Especially, for Shape-Net Car and AirfRANS, standard deviations on Spearman's rank correlations of drag or lift coefficients are provided.}
	\label{tab:mainres_standard_std}
	\vspace{-5pt}
	\vskip 0.15in
	\centering
	\begin{small}
		\begin{sc}
			\renewcommand{\multirowsetup}{\centering}
			\setlength{\tabcolsep}{1.2pt}
			\begin{tabular}{l|cccccccc}
				\toprule
                    \multirow{3}{*}{Model ($\times 10^{-2}$)} & \multicolumn{1}{c}{Point Cloud} & \multicolumn{3}{c}{Structured Mesh} & \multicolumn{2}{c}{Regular Grid} & \multicolumn{2}{c}{Unstructured Mesh} \\
                    \cmidrule(lr){2-2}\cmidrule(lr){3-5}\cmidrule(lr){6-7}\cmidrule(lr){8-9}
				& Elasticity & Plasticity & Airfoil & Pipe & \scalebox{0.9}{Navier–Stokes} & Darcy & \scalebox{0.9}{Shape-Net Car} & AirfRANS \\
				\midrule
                    \multirow{2}{*}{Second-best Model} & 0.86\scalebox{0.7}{$\pm 0.02$} & 0.17\scalebox{0.7}{$\pm 0.01$} & 0.59\scalebox{0.7}{$\pm 0.01$} & 0.47\scalebox{0.7}{$\pm 0.02$} & 11.95\scalebox{0.7}{$\pm 0.20$} & 0.65\scalebox{0.7}{$\pm 0.01$} & 98.42\scalebox{0.7}{$\pm 0.12$} & 99.64\scalebox{0.7}{$\pm 0.07$} \\
                    & \scalebox{0.8}{(GNOT)} & \scalebox{0.8}{(OFormer)} & \scalebox{0.8}{(LSM)} & \scalebox{0.8}{(GNOT)} & \scalebox{0.8}{(ONO)} & \scalebox{0.8}{(LSM)} & \scalebox{0.8}{(3D-GeoCA)} & \scalebox{0.8}{(GraphSAGE)} \\
                    \midrule
                    Transolver & \textbf{0.64}\scalebox{0.7}{$\pm 0.02$} & \textbf{0.12}\scalebox{0.7}{$\pm 0.01$} & \textbf{0.53}\scalebox{0.7}{$\pm 0.01$} & \textbf{0.33}\scalebox{0.7}{$\pm 0.02$} & \textbf{9.00}\scalebox{0.7}{$\pm 0.13$} & \textbf{0.57}\scalebox{0.7}{$\pm 0.01$} & \textbf{99.35}\scalebox{0.7}{$\pm 0.10$} & \textbf{99.78}\scalebox{0.7}{$\pm 0.04$} \\
                    Confidence Interval & 99\% & 99\% & 99\% & 99\% & 99\% & 99\% & 99\% & 95\% \\
				\bottomrule
			\end{tabular}
		\end{sc}
	\end{small}
\end{table*}

\section{Full Efficiency Analysis}\label{appdix:efficiency}
As a supplement to Figure \ref{fig:efficiency} in the main text, we also conduct experiments on different sizes of input meshes and record the model parameter, running time and the GPU memory of five Transformer-based methods, which are Transolver, ONO, GNOT, OFormer, Galerkin. From Figure \ref{fig:full_efficiency} and Table \ref{tab:all_efficiency}, we can find that in comparison with other methods, Transolver presents the least running time growth, which benefits from our design of linear-complexity Physics-Attention. Especially in large-scale meshes, Transolver is nearly 5x times faster than method ONO with $23\%$ the GPU memory usage.

\begin{table*}[t]
  \vspace{-5pt}
  \caption{Model efficiency comparison in Elasticity (Relative L2) and Shape-Net Car ($\rho_{D}$), where we select five Transformer-based methods that can be applied to unstructured meshes. Efficiency is evaluated on inputs of different meshes during training. Running time is measured by the time to complete one epoch, which contains $10^3$ iterations. ``/'' indicates the baseline will fail in this benchmark.}\label{tab:all_efficiency}
  \vskip 0.05in
  \centering
  \begin{threeparttable}
  \begin{small}
  \renewcommand{\multirowsetup}{\centering}
  \setlength{\tabcolsep}{7pt}
  \begin{tabular}{c|c|ccc|c|c}
    \toprule
    \multicolumn{2}{c}{Input Mesh Size ($N$)} & Parameter & GPU Memory & Running Time & Elasticity  & Shape-Net Car   \\
    \cmidrule{6-7}
    \multicolumn{1}{c}{Model} & \multicolumn{1}{c}{$N$} & (MB) & (GB) & (s~/~epoch) & (972 mesh points) & (32,186 mesh points) \\
    \toprule
     & 1024& 0.9284 & 0.64 & 38.183  & \multirow{5}{*}{0.0064} & \multirow{5}{*}{0.9935} \\
    & 2048& 0.9284 & 0.69 & 38.678& & \\
    \textbf{Transolver} & 4096 &0.9284 & 0.80 & 39.011 & &  \\
     \textbf{(Ours)}& 8192 & 0.9284 & 1.02 & 39.048 & &  \\
     & 16384 & 0.9284 & 1.51 & 54.104 & &  \\
     & 32768 & 0.9284 & 2.36 & 98.552 & &  \\
     \midrule
     & 1024 & 5.2477 & 0.85& 54.282 & \multirow{5}{*}{0.0086} & \multirow{5}{*}{0.9833}\\
     & 2048 & 5.2477 & 1.07& 55.939& &  \\
     GNOT & 4096 & 5.2477 & 1.47& 60.857& &  \\
     \citep{hao2023gnot}& 8192 & 5.2477 & 2.33& 67.170& &  \\
     & 16384 & 5.2477& 4.23& 112.552& &  \\
     & 32768 & 5.2477 & 7.46 & 209.923 & &  \\
    \midrule
     & 1024 & 1.1093 & 1.47 & 69.759 & \multirow{5}{*}{0.0118} & \multirow{5}{*}{/}  \\
     & 2048 & 1.1093 & 1.75& 76.245& &  \\
     ONO & 4096 & 1.1093 & 2.30& 100.134& &  \\
     \citep{anonymous2023improved} & 8192 & 1.1093 & 3.47& 149.598& &  \\
     & 16384 & 1.1093& 5.64& 255.339& &  \\
     & 32768 & 1.1093 & 10.09 & 462.459 & &  \\
    \midrule
     & 1024 & 0.8844 & 0.63& 28.147 & \multirow{5}{*}{0.0183} & \multirow{5}{*}{/}  \\
     & 2048 &0.8844 & 0.69& 30.983& &   \\
     OFormer & 4096 &0.8844 & 0.80& 31.113& &   \\
     \citep{li2023transformer}& 8192 &0.8844 & 1.02& 47.904 & &   \\
     & 16384 &0.8844 & 1.67& 91.671& &   \\
     & 32768 & 0.8844 & 2.44 & 182.205 & &  \\
    \midrule
     & 1024 & 1.0414 & 0.62& 26.507  & \multirow{5}{*}{0.0240} & \multirow{5}{*}{0.9764}   \\
     & 2048 &1.0414 & 0.66& 26.503& &  \\
     Galerkin Transformer & 4096 & 1.0414& 0.74& 27.481& &  \\
     \citep{Cao2021ChooseAT} & 8192 & 1.0414& 0.91& 37.098& &  \\
     & 16384 & 1.0414& 1.45& 67.524& &  \\
     & 32768 & 1.0414 & 2.05 & 129.872 & &  \\
    \bottomrule
    \end{tabular}
    \end{small}
    \end{threeparttable}
    \vspace{-5pt}
\end{table*}

\begin{figure*}[h]
\begin{center}
\centerline{\includegraphics[width=0.85\textwidth]{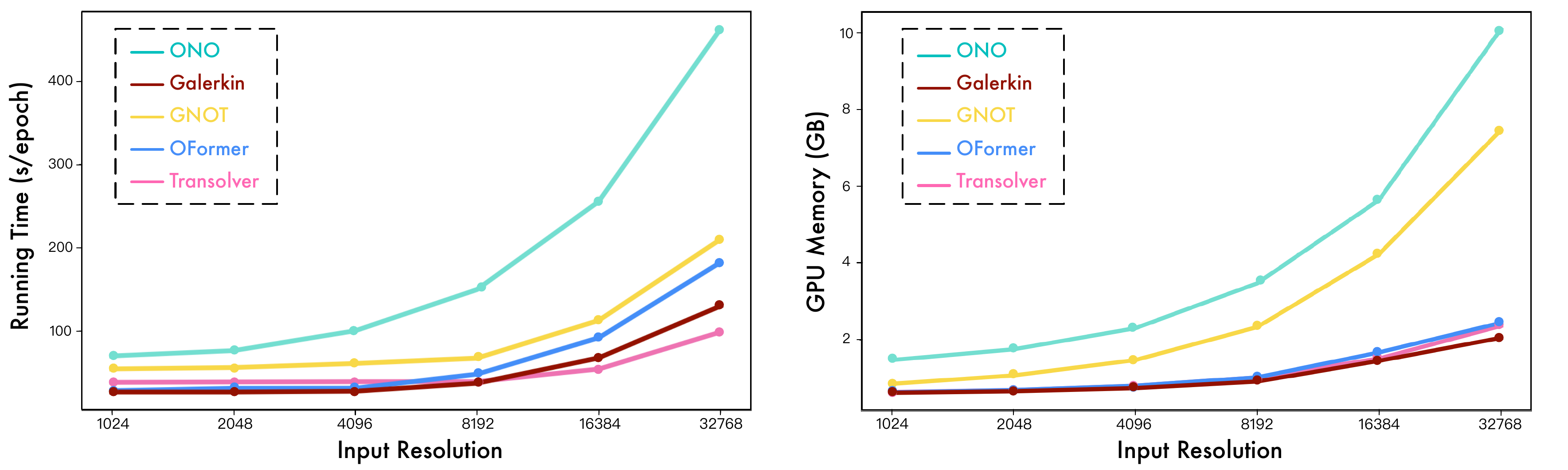}}
    \vspace{-10pt}
	\caption{The growth curves of Transformer-based models w.r.t.~the size of input mesh, where the batch size is set as $1$. We record the running time and the GPU memory under different mesh points $N$, which range from $2^{10}$ to $2^{15}$.}
	\label{fig:full_efficiency}
\end{center}
\vspace{-20pt}
\end{figure*}



\end{document}